\documentclass[11pt]{article}
\usepackage[T1]{fontenc}
\usepackage[
  letterpaper,
  textwidth=6.35in,
  textheight=8.5in,
  hcentering,
  vcentering
]{geometry}
\usepackage[title]{appendix}
\usepackage[sort,authoryear,round]{natbib}
\usepackage{amsmath,amssymb,amsthm,mathtools,dsfont,thm-restate}
\usepackage{microtype, float}
\usepackage{xcolor}
\usepackage{graphicx}
\usepackage{booktabs}
\usepackage{tabularx}
\usepackage{array}
\usepackage{enumitem}
\usepackage{nicefrac}
\usepackage{tikz}
\usetikzlibrary{arrows.meta,positioning,calc,fit,shapes.misc}
\definecolor{linkblue}{RGB}{0,91,181}
\usepackage[pagebackref,colorlinks=true,linkcolor=linkblue,citecolor=linkblue,urlcolor=linkblue]{hyperref}
\usepackage[capitalize,noabbrev]{cleveref}

\newtheorem{theorem}{Theorem}[section]
\newtheorem{proposition}[theorem]{Proposition}
\newtheorem{lemma}[theorem]{Lemma}
\newtheorem{claim}[theorem]{Claim}
\newtheorem{corollary}[theorem]{Corollary}
\newtheorem{definition}[theorem]{Definition}

\newtheorem{remark}[theorem]{Remark}
\newtheorem{problem}{Open Problem}
\newtheorem*{problem*}{Open Problem}

\crefname{theorem}{Theorem}{Theorems}
\crefname{proposition}{Proposition}{Propositions}
\crefname{lemma}{Lemma}{Lemmas}
\crefname{claim}{Claim}{Claims}
\Crefname{claim}{Claim}{Claims}
\crefname{corollary}{Corollary}{Corollaries}
\crefname{definition}{Definition}{Definitions}
\crefname{example}{Example}{Examples}
\crefname{remark}{Remark}{Remarks}
\crefname{openproblem}{Open Problem}{Open Problems}
\crefname{section}{Section}{Sections}
\crefname{subsection}{Section}{Sections}
\crefname{equation}{Equation}{Equations}

\crefname{appendix}{appendix}{appendices}
\Crefname{appendix}{Appendix}{Appendices}
\IfFileExists{macros.tex}{
  \renewcommand{\epsilon}{\varepsilon}

\newcommand{\N}{\mathbb{N}}
\newcommand{\Z}{\mathbb{Z}}
\newcommand{\E}{\mathbb{E}}
\newcommand{\F}{\mathbb{F}}
\renewcommand{\P}{\mathbb{P}}
\newcommand{\1}{\mathbf{1}}
\newcommand{\Ind}{\mathds{1}}
\renewcommand{\H}{\mathcal{H}}
\newcommand{\Hcal}{\mathcal{H}}
\newcommand{\Acal}{\mathcal{A}}
\newcommand{\Dcal}{\mathcal{D}}
\newcommand{\Fcal}{\mathcal{F}}

\newcommand{\Ecal}{\mathcal{E}}
\newcommand{\Xcal}{\mathcal{X}}
\newcommand{\Ycal}{\mathcal{Y}}

\newcommand{\Wcal}{\mathcal{W}}
\newcommand{\Lcal}{\mathcal{L}}
\newcommand{\Bcal}{\mathcal{B}}

\newcommand{\DS}{\mathrm{DS}}

\newcommand{\Unif}{\mathrm{Unif}}
\newcommand{\err}{\operatorname{err}}

\newcommand{\supp}{\operatorname{supp}}

\newcommand{\starlabel}{*}
\newcommand{\dollarlabel}{\mathord{\$}}

\newcommand{\clean}{\mathrm{iid}}
\newcommand{\obl}{\mathrm{oblivious}}
\newcommand{\adap}{\mathrm{adaptive}}
\newcommand{\semi}{\mathrm{semi}}
\newcommand{\ad}{\mathrm{ad}}

\newcommand{\prop}{\mathrm{proper}}
\newcommand{\imp}{\mathrm{improper}}
\newcommand{\Bin}{\operatorname{Bin}}

\newcommand{\Shuf}{\operatorname{Shuf}}

\newcommand{\edefn}[1]{\emph{\textbf{#1}}}
\providecommand{\formatvspace}[1]{\vspace{#1}}
\providecommand{\altvspace}[1]{}
\makeatletter
\@ifundefined{proofof}{
  \newenvironment{proofof}[1]{\renewcommand{\proofname}{Proof of #1}\proof}{\endproof}
}{}
\makeatother

\newcommand{\dom}{\operatorname{dom}}

}{
  
}

\usepackage{color-edits}
\addauthor[Julian]{j}{orange}
\addauthor[Chirag]{c}{purple}
\addauthor[Shaddin]{s}{blue}
\addauthor[Codex]{ai}{teal}

\newcommand{\email}[1]{\textsf{#1}}

\title{When Clean Data Hurts: Learning with Monotone \\ Corruptions Beyond Binary Classification}
\author{
Julian Asilis \\ USC \\ \email{asilis@usc.edu} \and
Shaddin Dughmi \\ USC \\ \email{shaddin@usc.edu} \and
Chirag Pabbaraju \\ Stanford University \\ \email{cpabbara@stanford.edu}
}
\date{}
\hypersetup{
  pdftitle={When Clean Data Hurts: Learning with Monotone Corruptions},
  pdfauthor={Julian Asilis, Shaddin Dughmi, and Chirag Pabbaraju}
}

\begin{document}
\maketitle

\begin{abstract}
Optimal learners are tailored to exploit the i.i.d.\ data assumption underlying the classic PAC model. What if an i.i.d.\ training sample were corrupted with correctly labeled examples drawn from an otherwise unrelated, even adversarial source? This model of learning with \emph{monotone adversarial corruptions} was recently introduced by \citet{og-paper}, who demonstrated that all known optimal learners for binary classification suffer increased error rates in this setting, from $O(d / n)$ in the PAC model to $\Omega (d \log(n / d) / n)$ under monotone corruption. \citet{mehrotra2026optimal} proved this logarithmic factor to be necessary in the monotone model for binary classification, but left open the consequences of corruption for more general learning settings, such as multiclass classification and partial binary concept classes \citep{alon2022theory}. 

As our primary result, we demonstrate that monotone adversaries are frighteningly more powerful in each of these settings. We exhibit a learnable multiclass problem, of DS dimension only 2, that becomes altogether \emph{unlearnable} under a monotone adversary, and show an analogous result for partial binary concept classes. These results are achieved by an adaptive adversary permitted to view the original i.i.d.\ training set $S$ and to insert $b < \infty$ corrupted datapoints into $S$. In the multiclass example, the adversary need only insert a linear number $b = |S| = n$ of datapoints.

We complement these impossibility results with characterizations of various settings in which monotone data is harmless --- or, at a minimum, cannot destroy learnability. We first prove that every ordinarily learnable class remains learnable when the number of adaptive additions is $o(n)$, which our previous multiclass lower bound proves to be tight. We then observe that the classic multiclass error rate of $O(d_{\mathrm{DS}} / n)$ remains achievable against adaptive adversaries restricted to a known constant budget $b = O(1)$, against semi-adaptive adversaries viewing only a $p$-fraction of $S$ for $p \in (0, 1)$, and against oblivious adversaries that cannot view $S$. Finally, we consider the landscape of proper multiclass learning, in which the learner must output a hypothesis in the underlying class. Here we find that even oblivious adversaries can produce arbitrarily large increases in the sample complexity of learning, underscoring the brittleness of common algorithmic templates such as Empirical Risk Minimization (ERM). 
\end{abstract}

\clearpage 
\begingroup
\hypersetup{linkcolor=black}
\setlength{\parskip}{0pt}
\setcounter{tocdepth}{2}
\linespread{1.0}\selectfont
\tableofcontents
\endgroup
\clearpage

\section{Introduction}\label{sec:introduction}

\formatvspace{-0.3 cm}
Contemporary machine learning happily exploits the data deluge: frontier models are trained on ever-increasing amounts of data, guided by empirical scaling laws that reward increases in model and data size \citep{grattafiori2024llama3,hoffmann2022training,kaplan2020scaling,
muennighoff2023scaling}. Classic models of statistical learning, however, typically invoke the assumption that training data is not just correctly labeled but furthermore generated i.i.d.\ from the ground truth process. This raises a fundamental question: 

\formatvspace{-0.2 cm}
\begin{quote} 
\begin{center} 
\emph{To what extent are classical results in learning theory sensitive to the addition of correctly-labeled but otherwise arbitrary, even adversarial training examples?} 
\end{center} 
\end{quote}
\formatvspace{-0.2 cm}

The model for studying such \emph{monotone adversarial corruptions} was introduced recently by \citet{og-paper}. In this framework, a target function $h^\star \in \Hcal$ first labels $n$ examples drawn i.i.d.\ from the marginal $D$. Subsequently, an adversary appends $b$ many additional examples that are also labeled by $h^\star$ but otherwise arbitrary, and the complete list of examples is shuffled. At test time, the learner is judged using only datapoints drawn from the true process $D$ and labeled by $h^\star$. Intuitively, the learner receives only ``clean'' data (i.e., correctly labeled by $h^{\star}$), but may have its understanding of the unlabeled data distribution $D$ warped by the adversary. 

For binary classification, \citet{og-paper} demonstrated that this loss of exchangeability in the training data strictly worsens the error rates of all known optimal learners, from the classical $O(d / n)$ rate, where $d$ denotes the VC dimension, to $\Omega(d \log( n / d) / n)$. This coincides with the uniform-convergence guarantee for VC classes, which gives error $O(d \log (n / d) / n)$ for ERM learners even in the face of monotone corruption. \citet{mehrotra2026optimal} then proved the logarithmic penalty to be necessary in the monotone setting, rendering ERM optimal. Together, these results provide a fairly complete account of monotone corruption for binary classification. 

What, however, are the consequences of monotone corruption for more general learning settings, such as multiclass classification and partial binary concept classes?\footnote{Recall that a partial binary concept is a function $h:\Xcal\to\{0,1,*\}$, and that realizable distributions are supported on $h^{-1}(\{0,1\})$ \citep{alon2022theory}. Thus, one is promised never to observe the undefined label $*$ at train or test time; the same restriction, naturally, is imposed upon the monotone adversary.} Precisely this question was raised in the original paper of \citet{og-paper}.

\begin{problem}[\cite{og-paper}]\label{open-prob}
In multiclass classification, does every classically learnable problem remain learnable under monotone corruption? In partial binary learning?
\end{problem}
\formatvspace{-0.3 cm}

It is worth noting that these settings can both differ significantly from (total) binary classification. On the positive side, they each have crisp combinatorial characterizations of learnability: by the DS dimension for multiclass learning and the partial VC dimension for partial binary learning \citep{DS14,brukhim2022characterization,alon2022theory}. 
However, learnability departs starkly from uniform convergence. That is, there exist classes that are learnable
yet for which interpolating hypotheses can fail to generalize. In particular, in both the multiclass and partial binary settings, there exist learnable problems that cannot be learned by ERM, or even by any proper learner.
As such, and unlike in the (total) binary case, it is conceivable that a monotone adversary could render 
a learnable problem altogther nonlearnable.

As our primary result, we demonstrate this to indeed be the case for both multiclass classification and 
binary partial-concept learning. Namely, in the multiclass setting, we exhibit a class of DS dimension only 2 whose 
learnability is destroyed by an adaptive monotone adversary making a linear number of insertions. (I.e., 
one corrupted datapoint is inserted for each of the i.i.d.\ examples.) Conversely, we demonstrate that 
every $o(n)$ adaptive budget preserves learnability, making the linear threshold sharp.
We further show that weaker adversaries can severely disrupt proper learning and ERM, even in settings where unrestricted learning remains essentially unaffected. We now describe these results.

\subsection{Results}\label{sec:intro-results}

We summarize our primary contributions as follows:

\begin{itemize}[leftmargin=1.5em,itemsep=0.8em]
    \item \textbf{Adaptive additions can destroy learnability.}
    We exhibit a countable multiclass problem $\H$ of DS dimension 2 that is properly learnable in the classic i.i.d.\ model, yet whose learnability can be destroyed by an adaptive monotone adversary making exactly $n$ additions to the original dataset of $n$ i.i.d.\ datapoints (\Cref{thm:multiclass-destruction}). Similarly, in partial binary learning, we construct a class of partial VC dimension at most 18 that becomes unlearnable under an adaptive adversary making $2^{O(n)}$ many insertions (\Cref{thm:partial-destruction}). These theorems resolve \cref{open-prob} negatively.

    \item \textbf{Restricted additions are harmless.} We next observe that if the learner is permitted $r$ many \emph{private} i.i.d.\ points that are not exposed to the adversary, then a one-inclusion learner achieves expected error at most $d_{\DS}/(r+1)$ (\Cref{thm:exchangeable-reserve}). Hence arbitrarily many oblivious additions preserve the optimal multiclass rate, while a semi-adaptive adversary that sees only $m$ of the $n$ clean examples incurs error at most $d_{\DS}/(n-m+1)$. For (fully) adaptive adversaries that are permitted to view all $n$ i.i.d.\ points, our previous results demonstrate that learnability can be destroyed, using only $b_n = n$ insertions in the multiclass case. Conversely, we use a random thinning argument to demonstrate that ordinary learnability is always preserved under $b_n = o(n)$ additions, establishing a sharp threshold between $o(n)$ and $\Theta(n)$ additions. Further, we show that a learner given an upper bound $b$ on the number of insertions can always achieve an error of $O(d_{\DS}(b + 1) / (n + b))$.

    \item \textbf{Properness and ERM are especially fragile.} Our previous results use somewhat convoluted, bespoke learners in order to retain learnability under oblivious and/or sublinear additions. We next consider the effects of monotone additions for more natural families of learners, such as \emph{proper} learners emitting only functions in the underlying class $\H$. We demonstrate that the picture is somewhat bleak: with $O(n/\delta)$ oblivious additions at confidence $\delta$, the sample complexity of proper learning can increase by arbitrarily large amounts (\Cref{thm:arbitrary-proper-blowup}). That is, a monotone sample complexity lower bound of $F(1 / \delta)$ can be achieved for any growth function $F$. We leave open the question of whether proper multiclass learnability can be destroyed altogether by an oblivious adversary.
    We also give a class of DS dimension 1 that can be learned by some ERM learners in the i.i.d.\ case, but not by any such learner after linearly many oblivious additions.
    
\end{itemize}

\subsection{Related work}\label{sec:related-work}

\textbf{Monotone adversarial corruptions.}
The monotone corruption model was introduced by \citet{og-paper}, who showed that the loss of exchangeability between training and test data can increase the error of all known optimal binary learners from the classical $O(d/n)$ rate to $\Omega(d\log(n/d)/n)$ by an adaptive adversary. \citet{mehrotra2026optimal} then proved this penalty to be unavoidable in the worst case, completing the binary picture.  Our results, in contrast, demonstrate that adaptive monotone additions can destroy learnability altogether in multiclass and partial-concept settings.

\altvspace{0.2 cm}
\noindent \textbf{Binary, multiclass, and partial-concept learning.}
Binary PAC learnability is classically characterized by the VC dimension, with optimal expected risk $\Theta(d/n)$ \citep{BEHW89,hanneke2016optimal}.  This rate is attained by several strikingly simple procedures, including Breiman's bagging predictor and a majority of only three ERMs trained on separate portions of the sample \citep{breiman1996bagging,larsen2023bagging,aden2023optimal,aden2024majority, rawal2026majority}.  On the other hand, multiclass learnability is characterized by the DS dimension \citep{DS14,brukhim2022characterization,pabbaraju2026optimal}, while learnability of binary partial concept classes is characterized by the partial VC dimension \citep{alon2022theory}.  Notably, in these broader settings, learnability need not imply uniform convergence, and some learnable multiclass classes admit neither a successful ERM nor any proper learner \citep{DS14,multiclassERM2015}.

\altvspace{0.2 cm}
\noindent \textbf{Related adversarial-data models.}
Monotone corruption bears some similarities to semirandom inference, robust statistics, malicious-noise learning, and clean-label poisoning, but differs in what the adversary may observe and alter.
In particular, semirandom models permit constrained adversarial modifications of a planted random instance, while robust statistics and malicious-noise learning allow genuinely corrupted observations \citep{feige2021semirandom,huber1964robust,kearns1993malicious}. The closest comparison is instance-targeted poisoning, where the attack may be tailored to a designated test point; clean-label variants additionally require inserted examples to carry correct labels \citep{blum2021clean,hanneke2022targeted}. In the monotone corruptions model, meanwhile, the adversary observes the training sample but \textit{not} the future test point, may only append correctly labeled examples, and is evaluated by ordinary population error on a fresh draw from the original marginal.

\section{Preliminaries}\label{sec:preliminaries}

Let $\H\subseteq\Ycal^{\Xcal}$ be a hypothesis class, let $h^\star\in\H$ be a target, and let $D$ be a marginal distribution on $\Xcal$. A clean sample of size $n$ is $S=((X_i,h^\star(X_i)))_{i=1}^n$, where $X_1,\ldots,X_n$ are i.i.d.\ from $D$. A monotone adversary appends a finite labeled multiset $U=((z_j,h^\star(z_j)))_{j=1}^b$ and gives the learner a uniformly shuffled copy of the multiset $T=S \cup U$. The learner is evaluated on the untouched marginal through
\[
    \err_{D,h^\star}(f)=\Pr_{X\sim D} \Big( f(X)\neq h^\star(X) \Big).
\]
A learner is \emph{proper} if it is restricted to always output a member of $\H$, and it is \emph{improper} if no such restriction applies. We will use standard facts from learning theory, namely that multiclass PAC learnability is characterized by the DS dimension \(d_{\mathrm{DS}}\) \citep{DS14,brukhim2022characterization}, while learnability of binary partial concept classes is characterized by the partial VC dimension \(\operatorname{PVC}\) \citep{alon2022theory}. We also make use of the following sharp one-inclusion theorem.

\begin{theorem}[\citet{pabbaraju2026optimal}]
\label{thm:oig-density}
If $d_{\DS}(\H)=d<\infty$, then every finite projection of $\H$ admits a
one-inclusion orientation for which all nodes have outdegree at most $d$.
\end{theorem}

The adversary is always made aware of $(\H,D,h^\star,n)$, and may select the unlabeled datapoints $z_j$ without any restrictions. It is \emph{oblivious} if it cannot view the realized dataset $S$, and \emph{fully adaptive} if it is granted full knowledge of $S$. For $m \le n$, an $m$-semi-adaptive adversary sees $m$ clean positions selected independently of their values and chooses its additions based on those observations. A budget-$b$ adversary must satisfy $|U|\le b$. 

For an adversary type $a$, a learner restriction $o\in\{\prop,\imp\}$, and a budget $b\in\N\cup\{\infty\}$, we write
\[
    R^o_{a,b}(n;\H)
    =\inf_A\sup_{D,h^\star,\mathsf{Adv}}
      \E \, \err_{D, h^\star}(A(\Shuf(S \cup U))),
\]
where $\Shuf(\cdot)$ is a random shuffle, the infimum ranges over learners of type $o$, the supremum ranges over budget-$b$ adversaries of type $a$, and the expectation is over the draw of $S \sim D^n$ (labeled by $h^\star$), the adversarial insertions $U$ and the random shuffle. We omit $b$ when it is unbounded, and use $R^o_{\clean}$ for ordinary i.i.d.\ learning. The corresponding high-probability sample complexity is denoted by $N^o_a(\epsilon,\delta;\H)$.
For partial concepts, $\H\subseteq\{0,1,\starlabel\}^{\Xcal}$ and we set $\dom(h)=\{x:h(x)\in\{0,1\}\}$. In this setting, a realizable pair $(D, h^\star)$ is required to satisfy $D(\dom(h^\star))=1$, and a monotone adversary may only insert non-$*$ examples, i.e., $z_j \in \dom(h^\star)$.

\section{Destroying Learnability with Adaptive Adversaries}\label{sec:destroying}

We now resolve \cref{open-prob} of \citet{og-paper} by demonstrating that adaptive adversaries can destroy learnability altogether in both the multiclass and partial binary settings. This stands in sharp contrast to the case of total binary classification, where arbitrary monotone additions nevertheless preserve learnability owing to uniform convergence, and can merely inflate sample complexities by a logarithmic factor at worst \citep{mehrotra2026optimal}.

\subsection{Multiclass learnability can be destroyed}
\label{sec:multiclass-destruction}

\begin{restatable}[Lower Bound for Multiclass Classes with a Monotone Adversary]{theorem}{MulticlassLowerBound}
\label{thm:multiclass-destruction}
There are countable discrete spaces $\Xcal, \Ycal$ and a countable class
$\H\subseteq \Ycal^{\Xcal}$ with $d_{\mathrm{DS}}(\Hcal) = 2$ such that:
\begin{enumerate}[label=(\roman*),leftmargin=2em]
    \item \textbf{Clean proper learning.} There is a proper learner that learns $\Hcal$. In particular,
    $$N^{\prop}_{\clean}(\epsilon, \delta; \H)=O\left(\frac{\log(1/\epsilon)+\log(1/\delta)}{\epsilon}\right), \qquad N^{\imp}_{\clean}(\epsilon,\delta;\H)=O\left(\frac{\log(1/\delta)}{\epsilon}\right).$$

    \item \textbf{Every learner foiled by some adaptive monotone adversary.} For every clean dataset size $n$, with an adaptive monotone corruption budget $b=n$, it holds that
    $$
        R^{\imp}_{\adap, b}(n;\H) \ge 1/4.
    $$
\end{enumerate}
\end{restatable}

\formatvspace{-0.6 cm}
\begin{proof}[Proof sketch]
Let us first define the class. Fix an even scale $r \in 2 \N$. Let
$\Bcal_r=\binom{[r]}{r/2}$ be the set of all halfsets of $[r]$, and let $K_r=\{0,1\}^{\Bcal_r}$. Here, we index the coordinates of the bit-string
$k=(k_B)_{B\in\Bcal_r}\in K_r$ by the halfsets $B\in\Bcal_r$.
For any \textit{tag} $t\in[r]$, a \emph{$t$-masked vector} is a partial bit-string
$a\in\{0,1,*\}^{\Bcal_r}$ satisfying $a_B=*$ if and only if $t\in B$.
Thus, the tag $t$ hides exactly those coordinates that are indexed by halfsets containing $t$,
while displaying a binary value in every other ``visible'' coordinate. Define 
\[
    \Xcal_r
    :=
    \left\{
        (r,t,a):
        t\in[r],\
        a\text{ is a $t$-masked vector}
    \right\},
    \qquad
    \Ycal_r
    :=
    \{(r,k):k\in K_r\},
\]
where $\Xcal_r$ is the domain at scale $r$. For $x=(r,t,a)\in\Xcal_r$ and $k \in K_r$, let
$\Delta_a(k)=\{B\in\Bcal_r:a_B\neq*\text{ and }a_B\neq k_B\}$ be the set
of visible coordinates at which $a$ disagrees with $k$. The hypothesis at scale $r$ that is indexed
by $k$ is defined on its home domain $\Xcal_r$ by
\[
    h_{r,k}(r,t,a)
    :=
    \begin{cases}
        (r,k\oplus e_B),
            & \Delta_a(k)=\{B\},\\
        (r,k),
            & \text{otherwise}.
    \end{cases}
\]
Outside $\Xcal_r$, the hypothesis simply outputs its own name $(r,k)$.
The entire domain is the disjoint union over $\Xcal_r$ at all scales $r$, and setting
$\H=\{h_{r,k}:r\ge2\text{ even},\,k\in K_r\}$, gives the required countable
spaces $\Xcal,\Ycal$ and $\H$.

In words, a hypothesis in the class announces its own name at most points in the domain. Only when the visible
display $a$ of the point contains exactly one incorrect bit $a_B \in \{0,1\}$, it instead announces the
neighboring bit-vector obtained by repairing that bit. The following observations enable proper learning fairly easily: 
Either only a single name $(r,k)$ appears in the training data, in which case the hypothesis $h_{r,k}$ is a consistent hypothesis. Otherwise, there are two distinct names $(r,w)$ and $(r,w')$ in the training data. In this case, either one of $w$ or $w'$ equals $k$, or the correct $k$ is at a Hamming distance of exactly $1$ from both of $w$ and $w'$ simultaneously, and there are only two strings with this property. Therefore, the proper learner can extract at most $N + 2 \binom{N}{2} = N^2$ candidate hypotheses by inspecting  singletons and pairs in the training data, and output a candidate that is consistent with the entire training data. A simple probability calculation combined with the union bound then establishes that this learner has vanishing error. In fact, while the union bound loses a logarithmic factor in the sample complexity of the proper learner, it can be directly argued that this class has constant DS dimension, and hence can be learned improperly with the stated optimal sample complexity. 

For the lower bound, set $r=2n$. For $k\in K_r$, let $x_t(k)=(r,t,a)$ be such that the visible entries in the $t$-masked vector $a$ agree with $k$, and let $D_k$ be
uniform on the $2n$ points $x_t(k)$, $t\in[2n]$. Compare $k$ with the
neighboring bit-string $w=k\oplus e_B$ which simply flips the bit at the $B^\text{th}$ coordinate, and consider the behavior of the two target hypotheses $h_{r,k}$ and $h_{r,w}$. If $t\in B$, then the coordinate in the corresponding $t$-masked vectors that would have been different is masked, and hence, $x_t(k)=x_t(w)$, although the two hypotheses assign distinct
labels $(r,k)$ and $(r,w)$ respectively to these points. On the other hand, if $t\notin B$, the two points becomes distinct,
but the label assignment rule of the hypotheses --- which corrects a discrepancy if it is unique --- makes both the hypotheses assign the
same label to each point.

Now, a clean sample of size $n$ has at most $n$ of the $2n$ distinct tags, so some halfset $B$
avoids every observed tag. The adversary may choose such a $B$, and, for each occurrence of $(x_t(k),(r,k))$ in the clean dataset, append the matched example
$(x_t(w),(r,w))$. Concretely, if $f_t$ is the number of clean observations with tag
$t$, the final ``monotone corrupted'' labeled multiset is
\[
    T_{k,f}
    :=
    \bigcup_{t \in [r]}
    f_t\cdot
    \left\{
        (x_t(k),(r,k)),
        (x_t(w),(r,w))
    \right\}.
\]
By our reasoning above, this multiset is correctly labeled by \textit{both} the hypotheses $h_{r,k}$ and $h_{r,w}$, 
and is unchanged when $k$ and $w$ exchange roles. 
The learner therefore receives the same input distribution in the two worlds. However, the target function from these two worlds
disagree on the common points indexed by $B$, which carry half of the test
mass. It follows immediately that no predictor can achieve small risk in both worlds (i.e., using only the information present in the corrupted sample). The argument is completed with an application of the probabilistic method, averaging over pairs of the form $(k, k \oplus e_B)$. See \Cref{sec:multiclass-lower-bound} for the complete details.
\end{proof}

\altvspace{-0.4 cm}
\subsection{Price of adaptivity at DS dimension one}
\label{sec:portal-main}

A curious ``phase transition'' in binary classification was recently
demonstrated by \citet{mehrotra2026optimal}: adaptive adversaries can inflate
sample complexities by a logarithmic factor for all classes of VC dimension
$d\ge2$, but remain harmless at $d=1$.  We now demonstrate that multiclass
classification exhibits no such transition---an adaptive adversary can
deteriorate the sample complexity even at DS dimension one. 

\begin{theorem}
\label{thm:portal-degradation}
There is a countable multiclass hypothesis class $\H$ with
$d_{\DS}(\H)=1$ such that
\[
    R^{\imp}_{\clean}(n;\H)
    =
    O(1/n),
    \qquad
    R^{\imp}_{\ad,n}(n;\H)
    =
    \Omega(\log n/n).
\]
\end{theorem}

\formatvspace{-0.6 cm}
\begin{proof}[Proof sketch]
For a prime number $q$, let $G_q=(L_q,R_q,E_q)$ be a $q$-regular,
$C_4$-free bipartite graph whose edges are properly colored by
$[q]$ (simple families of such graphs exist). The edges in the graph constitute the input domain, and there is a hypothesis $h_v$ corresponding to every vertex $v$. Any edge $e=uw$ with left vertex $u$ and right vertex $w$ is labeled by $h_v$ according to the following rule:

\[
    h_v(e)
    =
    \begin{cases}
        0,&u=v \text{ or } v \in N(u) \setminus \{w\},\\
        1,&v=w \text{ or } v \in N(w) \setminus \{u\},\\
        \lambda_{e,v},& \text{otherwise}.
    \end{cases}
\]
Here, $N(v)$ denotes the neighborhood of vertex $v$, and $\lambda_{e,v}$ is a label private to the pair $(e,v)$. Let $E(v)$ denote the set of edges incident to $v$. The crucial property of this class is the following: for any edge $e=uv$, $h_u(e) \neq h_v(e)$, but for any other edge $e' \in E(u) \setminus \{e\}$, it holds that $h_u(e')=h_v(e')=0$, and also for any edge $e' \in E(v) \setminus \{e\}$, it holds that $h_u(e')=h_v(e')=1$. We can directly argue that the DS dimension of this class is 1, since any pair of shattered points implies the existence of a 4-cycle in the graph, which is forbidden.


Now, for a target hypothesis corresponding to vertex $v$, let the clean marginal be uniform on the $q$ distinctly colored edges edges incident to $v$.  Because the edge coloring is
proper, the colors in a clean sample of size $n$ are equivalent to $n$ i.i.d\ draws from $[q]$.  
If some
color is absent, let $e=vw$ be the incident edge of that color.  The
adversary then mirrors every observed edge in the neighborhood of $v$ to the edge of the
\textit{same} color in the neighborhood of $w$. Under this corruption scheme, the two endpoint worlds produce the
same labeled multiset, but crucially disagree at the unseen edge $vw$, which has
test mass $1/q$.  By a coupon-collector bound, a missing color can be made to occur with constant probability for
$n\lesssim q\log q$; choosing $q\asymp n/\log n$ thus yields adaptive error
$\Omega(\log n/n)$ using at most one mirrored insertion per clean example.
The clean error remains $O(1/n)$ because the class has DS dimension one. The details appear in \Cref{app:portals}.
\end{proof}

\begin{remark}[Tensorization and general budgets]
    \label{remark:tensorization-general-budgets}
    By taking a product of $d$ independent copies of the above class, we can similarly obtain a class having DS dimension $d$ where any learner suffers error $\Omega(d\log(n/d)/n)$ with $n$ adaptive monotone corruptions. The lower bound can also be generalized to work for a given corruption budget $b$ in a straightforward manner, where with $b$ corruptions, the error of the learner can be made to be $\Omega(d\log(b/d)/n)$.
    For $d\lesssim b\le n$, give each component total mass $cb/n$ for a sufficiently small constant $c$ and place the remaining mass on a dummy point shared by all hypotheses; then, with constant probability, at most $b$ observations require mirroring, while choosing $q\asymp (b/d)/\log(b/d)$ yields the stated error.
\end{remark}

\subsection{Binary partial-concept learnability can be destroyed}
\label{sec:partial-destruction}

\begin{restatable}[Lower Bound for Binary Partial Classes with a Monotone Adversary]{theorem}{PartialLowerBound}
\label{thm:partial-destruction}
There is a countable discrete space $\Xcal$ and a countable class
$\H\subseteq\{0,1,\starlabel\}^{\Xcal}$ such that:
\begin{enumerate}[label=(\roman*),leftmargin=2em]
    \item \textbf{Clean learning.} The partial VC dimension of $\H$ is at most 18, and hence, $$N^{\imp}_{\clean}(\epsilon,\delta;\H)=O(\log(1/\delta)/\epsilon).$$

    \item \textbf{Every learner foiled by some adaptive monotone adversary.} For every clean dataset size $n$, with an adaptive monotone corruption budget $b=2^{O(n)}$, it holds that
    $$
        R^{\imp}_{\adap, b}(n;\H) \ge 1/4.
    $$
\end{enumerate}
\end{restatable}

\formatvspace{-0.6 cm}
\begin{proof}[Proof sketch]
The key idea is to use the technique from the multiclass construction, but to maintain partial binary learnability by
replacing the neighboring pair
$(k,k\oplus e_B)$ with an entire cyclic line of partial binary targets. Begin by, once again, fixing an even scale $r \in 2\N$ and setting $\Bcal_r=\binom{[r]}{r/2}$.

Then, choose a finite field $\F_q$ with $q > r + |\Bcal_r|$, and select distinct field elements $z_t$, $t\in[r]$, and $\alpha_B$, $B\in\Bcal_r$. Upon fixing such choices, we set $\phi_B(z)=z-\alpha_B$. 
We next consider a special set arranged as a finite \emph{marker flower} at each scale $r$. That is, for every ``direction'' $B\in\Bcal_r$, create a ``petal''
$E_B=\{e_{B,a}:a\in\F_q\setminus\{0\}\}$, and identify the missing element
$e_{B,0}$ of every petal with one common marker $\perp$. Then, the marker flower at scale $r$ is
\[
    \Omega_r
    :=
    \{\perp\}
    \sqcup
    \bigsqcup_{B\in\Bcal_r}E_B.
\]
Now, for a field element $a\in\F_q$, let $\tau_{B,a}$ be the map the rotates the petal in the direction $B$ by an amount $a$, i.e., $\tau_{B,a}(e_{B,c})=e_{B,c+a}$, and leaves petals in the other directions untouched, i.e., $\tau_{B,a}(e_{C,b})=e_{C,b}$ whenever $C \neq B$. 

Some more terminology to complete the specification of the class: Define a \emph{state} to be a pair $w=(\sigma,p)$, where $\sigma$ is a bijection over
$\Omega_r$ and $p(z)=u+vz$ is an affine function over $\F_q$. These states will index the hypotheses in the level-$r$ block.
For a \emph{tag} $t\in[r]$, let $ O_t := \bigsqcup_{B\not\ni t}E_B$
constitute all the petals in directions $B$ that \textit{do not} contain $t$. A \emph{$t$-cell} is then an injective function $F:O_t\hookrightarrow\Omega_r$. We say that a state $w=(\sigma,p)$ is correctable at a $t$-cell $F$ if either $\sigma|_{O_t}=F$, or if for some $B\not\ni t$ and $a\neq0$, undoing a rotation along direction $B$ by an amount $a$ makes $\sigma$ project to $F$, i.e., $(\sigma\circ\tau_{B,-a})|_{O_t}=F$. In the former case, the corrected state $w'$ is simply $w$ itself, whereas in the latter case, the corrected state is 
$$w'=(\sigma', p')=(\sigma\circ\tau_{B,-a},\,p-a\phi_B).$$
Importantly, the injectivity of $F$ makes the correction \textit{unique} whenever it exists. 

We are now ready to specify the domain: an unlabeled datapoint in the level-$r$ block is a tuple
$(r,t,F,P)$, where $t \in [r]$, $F$ is a $t$-cell and
$P\subseteq\F_q$ has size two. For every such $P$, fix a bijection
$\lambda_P:P\to\{0,1\}$. The hypothesis indexed by a state $w$ is
\[
    h_{r,w}(r,t,F,P)
    :=
    \begin{cases}
        \lambda_P\bigl(p'(z_t)\bigr),
            & \text{if $w$ correctable at $F$ with corrected state $w'=(\sigma', p')$, and }
              p'(z_t)\in P,\\
        \starlabel,
            & \text{otherwise}.
    \end{cases}
\]
That is, we first check if the state $w$ is correctable at the $t$-cell $F$; if so, we consider the corrected state $(\sigma', p')$, and check if $p'(z_t) \in P$, i.e., whether $p'(z_t)$ belongs to the domain of $\lambda_P$. If this is indeed the case, the label is set to be $\lambda_P(p'(z_t))$.
 Extending every $h_{r,w}$ by
$\starlabel$ outside its home block, and taking the disjoint union over all
even $r$, gives the fixed countable partial class $\H$.

For the lower bound at sample size $n$, we work in the block $r=2n$. For a state $w=(\sigma,p)$, write $F_t(w)=\sigma|_{O_t}$ and $c_{w,t}=p(z_t)$. The hard marginal $D_w$ chooses a tag $t$ uniformly, chooses $d$ uniformly from $\F_q\setminus\{0\}$, and outputs the datapoint
\[
    x_{t,d}(w)
    :=
    (r,t,F_t(w),\{c_{w,t},c_{w,t}+d\}),
\]
with label $\lambda_{\{c_{w,t},c_{w,t}+d\}}(c_{w,t})$.

Then, for any direction $B\in\Bcal_r$, if we consider the cyclic $B$-line through $w$, i.e.,
\[
    L_B(w)
    :=
    \left\{
        (\sigma\circ\tau_{B,a},\,p+a\phi_B):
        a\in\F_q
    \right\},
\]
we have that: (1) for any tag $t\notin B$, the petal rotation within the $B$-line may be corrected away, so every member within the line agrees with the example, (2) for any tag $t\in B$, the petal rotation is invisible: the members of the line induce the same $t$-cell, but their affine counterparts $p(z_t)+a\phi_B(z_t)$ enumerate the entire field. These properties ensure that no predictor may simultaneously have small error across every target $h_{r,v}$ and distribution $D_v$, when $v$ ranges over the $B$-line. 

The adversary exploits this: after observing the clean dataset, the aim of the adversary is to confuse the identity of the ground truth among all the members of a $B$-line. Concretely, for every labeled point $(x_{t,d}(w), h_{r,w}(x_{t,d}(w)))$ in the clean dataset, the adversary appends a copy $(x_{t,d}(v), h_{r,v}(x_{t,d}(v)))$ for every member $v \in L_{B}(w)$. The properties above guarantee that this is a valid monotone corruption. 
More crucially, the final labeled multiset (upon shuffling) is distributionally identical, irrespective of which member of the $B$-line corresponded to the ground truth. 
Since by our previous reasoning, a learner is forced to perform poorly with respect to at least some member of the $B$-line, we can extract this member out to be the ground truth, which gives us the lower bound.

Finally, the partial VC dimension of this class can be shown to be at most 18 (i.e., a constant). The central reason for this is the condition for any hypothesis $h_{r, (\sigma, p)}$ to be non-$*$ at a point $x=(r,t,F,P)$: the state $(\sigma,p)$ should be correctable at $F$, and furthermore, $p'(z_t)$ should belong to $P$ for the corrected state $(\sigma', p')$. Since $P$ is only a two-element set, and every $p$ is also an affine function in just two variables, this severely restricts the number of distinct binary patterns that can be realized by the class on any candidate shattered set $S$. The complete construction, proof, and analysis of the partial VC dimension appear in \Cref{sec:partial-class-lower-bound-rewrite}.
\end{proof}


\section{When Monotone Additions Are Harmless}\label{sec:harmless}

The lower bounds of \Cref{sec:destroying} apply to adaptive adversaries that are permitted to condition on the entire training set $S$ before selecting their corrupted insertions. What can be said of oblivious adversaries that are not permitted to view $S$, or of \emph{semi-adaptive} adversaries that can view only a fraction of training points in $S$? We now demonstrate that such adversaries are considerably less powerful. In particular, any adversary that observes only a constant fraction of $S$ cannot deteriorate the optimal expected error rate of $O(d_{\mathrm{DS}} / n)$. This generalizes Theorem 3.2 of \citet{og-paper}, which establishes such behavior for oblivious adversaries in binary classification.

\subsection{Oblivious and semi-adaptive adversaries}\label{sec:exchangeable-main}

\begin{definition}\label{def:semi-adaptive}
Fix $m\le n$. An \edefn{$m$-semi-adaptive adversary} is shown $m$ clean labeled examples, whose indices are selected independently of their values, and may choose its additions as an arbitrary function of these visible examples. The remaining $n-m$ clean examples are hidden.
\end{definition}

By exchangeability, revealing a fixed prefix and revealing a uniformly random $m$-subset define equivalent experiments. This is importantly different from allowing the adversary to inspect all $n$ clean examples and only afterward decide which $m$ of them count as visible.

The key observation is that hidden clean examples retain all of the symmetry of ordinary i.i.d.\ data. The adversary may surround them with arbitrarily many correctly labeled insertions, and the learner need not know which examples are hidden; nevertheless, these clean points remain exchangeable with a fresh test point. One-inclusion prediction needs precisely this symmetry---and surprisingly little else.

\altvspace{-0.25 cm}
\begin{theorem}[Exchangeable clean reserve]
\label{thm:exchangeable-reserve}
Let $d_{\DS}(\H)=d<\infty$. Suppose that among the examples received by the learner are $r$ clean examples $Z_1,\ldots,Z_r$, and let $Z_0$ be a fresh clean example. Assume that, after fixing all remaining examples and the clean information used to choose them, the examples $Z_0,Z_1,\ldots,Z_r$ remain exchangeable. Then the canonical one-inclusion learner satisfies
\[
    \E L_D(\widehat h,h^\star)
    \le
    \frac{d}{r+1}.
\]
The learner need not be told which observed examples constitute the exchangeable reserve.
\end{theorem}

\altvspace{-0.25 cm}
A useful way to view the proof is to rotate the role of the test point. Fix all other examples and the complete realization of $Z_0,\ldots,Z_r$, and withhold each of these $r+1$ examples in turn. The resulting predictions belong to one common one-inclusion problem, so the sharp one-inclusion guarantee allows at most $d$ mistakes across the $r+1$ possible held-out coordinates. Exchangeability makes the fresh test point no more likely to be one of these mistakes than any reserve example, yielding the factor $1/(r+1)$. The formal argument appears in \Cref{app:exchangeable-reserve}.

As a particular instantiation, note that oblivious adversaries cannot disturb this symmetry. Once the complete multiset of additions is fixed, the $n$ clean training examples remain exchangeable with a fresh test point, regardless of how many examples were inserted. We therefore obtain the following extension of the binary result of \citet{og-paper}: arbitrary oblivious additions preserve the optimal multiclass rate.

\altvspace{-0.25 cm}
\begin{corollary}
\label{cor:oblivious-free}
For any number of oblivious monotone additions,
$R^{\imp}_{\obl}(n;\H)\le d_{\DS}(\H)/(n+1)$.
\end{corollary}

\altvspace{-0.15 cm}
The same argument degrades gracefully as the adversary is shown more of the sample. After fixing the $m$ visible examples and every insertion chosen from them, the remaining $n-m$ clean examples are still exchangeable with a fresh test point.

\altvspace{-0.25 cm}
\begin{corollary}
\label{cor:semi-adaptive}
Against an $m$-semi-adaptive adversary,
$R^{\imp}_{m\text{-}\semi}(n;\H)\le d_{\DS}(\H)/(n-m+1)$.
In particular, classes of finite DS dimension remain learnable whenever the hidden reserve size $n-m_n$ diverges.
\end{corollary}

\altvspace{-0.15 cm}
Thus an invisible clean example remains, in a precise sense, fully valuable: each hidden point contributes one unit to the effective sample size of the one-inclusion learner. In particular, if the adversary observes at most a fixed $p$-fraction of the sample, where $p<1$, the optimal $O(d_{\DS}/n)$ rate survives up to the constant factor $1/(1-p)$.

\altvspace{-0.3 cm}
\subsection{Sublinear adaptive additions preserve learnability}\label{sec:thinning-main}

We now return to fully adaptive adversaries, for which the techniques of the previous section offer little hope. We demonstrate, however, that an adaptive adversary constrained to a sublinear number of insertions still cannot destroy learnability. The proof uses an embarrassingly simple learner: retain a small
random subsample and hope that every selected point is clean.

We write $R_{\clean}(A_t,t)$ for the worst-case expected risk of $A_t$ over realizable distributions.

\altvspace{-0.25 cm}
\begin{theorem}[Random-thinning robustification]\label{thm:thinning}
Let $A_t$ be any ordinary learner using $t\le n$ examples, and suppose the
adversary adds at most $b$ examples. Select $t$ positions uniformly without
replacement from the final sample and run $A_t$ on them. Then
\[
    R_{\ad,b}(n)
    \le
    R_{\clean}(A_t,t)
    +
    1-\frac{\binom nt}{\binom{n+b}{t}}
    \le
    R_{\clean}(A_t,t)+\frac{bt}{n+b}.
\]
The robust learner is proper whenever $A_t$ is proper.
\end{theorem}

\altvspace{-0.25 cm}
Remarkably, this crude upper bound combines with the masked-table
construction to yield a sharp qualitative dichotomy.

\altvspace{-0.25 cm}
\begin{corollary}
\label{cor:universal-threshold}
Every sequence $b_n=o(n)$ of adaptive monotone additions preserves ordinary PAC learnability, and preserves proper PAC learnability whenever it is achievable from clean data. Conversely, there is a fixed properly learnable multiclass class of DS dimension two for which $b_n=n$ additions force constant minimax error.
\end{corollary}

\altvspace{-0.2 cm}
Indeed, choose $t_n\to\infty$ slowly enough that both the clean risk of $A_{t_n}$ and $b_nt_n/(n+b_n)$ vanish. The argument treats the clean learner as a black box, and it preserves properness automatically.

\altvspace{-0.3 cm}
\subsection{Known-budget robustification}\label{sec:plurality-main}

We now consider the case in which the learner is informed of an upper bound
on $b$, and demonstrate that this knowledge empowers more intelligent
learners. In particular, one such learner distributes the sample across many
groups and aggregates their predictions. The crucial idea is that each
insertion can contaminate only one group.

\altvspace{-0.25 cm}
\begin{theorem}[Partition and plurality]\label{thm:plurality}
Assume that the adversary adds at most $b$ examples. Put $r=4b+1$ and
$t=\lfloor(n+b)/r\rfloor$. After padding to the declared final length with
correctly labeled duplicates, select $rt$ positions uniformly if necessary,
partition them randomly into $r$ groups of size $t$, run an ordinary improper
learner on every group, and output their pointwise plurality. Then
\[
    R^{\imp}_{\ad,b}(n;\H)
    <
    4 \, R^{\imp}_{\clean}(t;\H).
\]
Consequently, if $d_{\DS}(\H)=d$, then
$R^{\imp}_{\ad,b}(n;\H)=O(d(b+1)/(n+b))$.
\end{theorem}

\altvspace{-0.25 cm}
Intuitively, the reason this aggregation succeeds is that the adversary's inserted points can only contaminate so many groups. In particular, at most $b$ of the $4b+1$ groups are corrupted. Thus, if the plurality prediction errs at a test point, a constant fraction of the clean-group learners must have been incorrect. Averaging this deterministic voting inequality reduces the robust risk to the ordinary risk of one clean-group learner, up to a factor smaller than four. The DS-dimension consequence then follows from the optimal clean one-inclusion rate.

The same partition supports a useful list-valued variant. Split the sample into $b+s$ groups and output, at each test point, the set of labels proposed by the group learners. At least $s$ groups are clean, so the correct label is absent only if all of these clean learners fail. Taking $s=\lceil\log_2(1/\delta)\rceil$ yields list size $b+O(\log(1/\delta))$ using
\[
    O\!\left(
        \frac{d_{\DS}}{\epsilon}
        \left[
            b+\log\frac1\delta
        \right]
    \right)
\]
clean examples. The complete plurality and list arguments appear in \Cref{app:plurality}.

\section{Proper Learning and ERM}\label{sec:structured}
The results of \Cref{sec:harmless} place no restrictions on the form of the
successful learner. We now show that, once attention is restricted to proper
learners, even oblivious adversaries become considerably more powerful. In
particular, oblivious additions can make proper learning arbitrarily
expensive and can destroy every ERM, despite remaining harmless to
unrestricted learning. We then sharpen the latter phenomenon still further:
one correctly labeled adaptive addition already suffices to impose the
logarithmic binary ERM rate.

The classes witnessing our first two results are also topologically tame:
they are countable, closed, and have \emph{pointwise finite range}, meaning
that $\{h(x):h\in\Hcal\}$ is finite for every $x\in\Xcal$.

\subsection{Proper learning can become arbitrarily expensive}
\label{sec:proper-main}

\begin{theorem}[Arbitrarily severe proper degradation]
\label{thm:arbitrary-proper-blowup}
Let $F:[1,\infty)\to[1,\infty)$ be nondecreasing. There is a hypothesis class $\Hcal_F$ such that, for $\epsilon_0=1/16$,
\[
    N^{\prop}_{\clean}(\epsilon_0,\delta;\Hcal_F)
    =
    \Theta\!\left(\log\frac1\delta\right),
    \qquad
    N^{\imp}_{\ad}(\epsilon_0,\delta;\Hcal_F)
    =
    \Theta\!\left(\log\frac1\delta\right),
\]
whereas $N^{\prop}_{\obl}(\epsilon_0,\delta;\Hcal_F)\geq F(1/\delta)$ for all sufficiently small $\delta$.
\end{theorem}

\formatvspace{-0.6 cm}
\begin{proof}[Proof sketch]
Let $(L_k)_{k \geq 1}$ be a strictly increasing sequence of positive integers, to be chosen later
as a function of $F$. Make disjoint finite blocks
\[
 X_k=P_k\sqcup C_k,\qquad
 P_k=\{p_{k,1},\ldots,p_{k,k}\},\qquad |C_k|=2L_k.
\]
We refer to $C_k$ as the \emph{core}, and to points in $P_k$ as \emph{markers}. Set
$\Xcal=\bigsqcup_{k\geq1}X_k$. The $k$th block $X_k$ gives rise to two kinds of hypotheses:
\begin{align}
 c_{k,A}(x)&=
 \begin{cases}
 \star,&x\in P_k\cup A,\\
 \lambda_{k,A},&x\in C_k\setminus A,\\
 \dollarlabel,&x\notin X_k,
 \end{cases}
 &&\hspace{-1em} A\subseteq C_k,\quad |A|=L_k; \label{eq:set}\\[0.2em]
 s_{k,i}(x)&=
 \begin{cases}
 \rho_{k,i},&x=p_{k,i},\\
 \star,&x\in X_k\setminus\{p_{k,i}\},\\
 \dollarlabel,&x\notin X_k,
 \end{cases}
 &&\hspace{-1em} i\in[k]. \label{eq:fallback}
\end{align}
All $\lambda_{k,A}$ and $\rho_{k,i}$ are distinct private labels. Let $\Hcal$ consist
of such hypotheses over all blocks, along with the constant function $h_{\dollarlabel}\equiv\dollarlabel$.
On the core $C_k$, the hypotheses $c_{k,A}$ form a restriction of the first Cantor
class \citep{DS14}: they predict $\star$ on one half and a private label on the
other. Each \emph{fallback} $s_{k,i}$ instead predicts $\star$ throughout $X_k$
except at its marker $p_{k,i}$.

With i.i.d.\ data, proper learning is fairly uncomplicated. Either the sample $S$ contains a private label,
in which case the true labeling function has been revealed, or $S$ contains only $\dollarlabel$ and/or 
$\star$ labels. (Should the $\star$ label appear, it can only appear on a single block $X_k$.) 
In the latter case, a proper learner can safely output either $h_{\dollarlabel}$ (if all labels
are $\dollarlabel$)
or a fallback $s_{k, i}$, where $p_{k, i}$ is a least-frequently observed marker in $X_k$. Similarly,
under arbitrary adaptive additions, an improper learner can predict
$\star$ throughout $X_k$ and $\dollarlabel$ elsewhere. Its disagreement region is
precisely the target's private-label region. Since no private label appears in the
augmented sample, this entire region was missed by the original i.i.d.\ sample.
Consequently, the probability of incurring error greater than $\epsilon$ is at
most $e^{-n\epsilon}$, yielding the logarithmic confidence bound for adaptive
improper learning.

For proper learning in the face of an oblivious adversary, fix a block $k$ and choose
$A\in\binom{C_k}{L_k}$ and $J\in[k]$ uniformly at random. Then let $c_{k,A}$ be the true 
labeling function and use the marginal $D_{J, A}$ that places $1/2$ mass on $p_{k, J}$ and distributes
the remaining $1/2$ mass uniformly across $A$. More generally, we denote 
\[
    D_{i,A}=\tfrac12\delta_{p_{k,i}}+\tfrac12\Unif(A).
\]
For every $i\neq J$, an oblivious adversary adds an independent $n$-sample from
$D_{i,A}$. The final transcript therefore contains one batch from each of the $k$
marginals, and its distribution does not reveal which index $J$ generated the
clean batch. A fallback $s_{k,i}$ has error $1/2$ when $i=J$, so any selected
fallback incurs large error with probability $1/k$. Should the learner instead 
output a Cantor hypothesis $c_{k,B}$, it would need to approximate the set $A$ given 
only $O(kn)$ draws from $\Unif(A)$. When $kn\leq L_k/2$, the learner simply has 
insufficient information to do so---much as in the classic first Cantor class of 
\citet{DS14}---forcing constant error. All other hypotheses in $\Hcal$ emit $\dollarlabel$ on 
the block $X_k$ and thus incur error one.  Thus
every proper learner fails with probability at least $1/k$ whenever
$n\leq L_k/(2k)$. Taking $k\asymp1/\delta$ and choosing
$L_k\geq2k\lceil F(2(k+1))\rceil$ gives the desired bound. 
The complete arguments appear in \Cref{app:proper-blowup}.
\end{proof}

\subsection{ERM can fail completely}
\label{sec:erm-main}

\begin{theorem}[Oblivious additions destroy ERM]
\label{thm:erm-fragile}
There is a class $\Hcal$ of DS dimension $1$ that is ordinarily PAC learnable by
a proper ERM. Nevertheless, for every fixed $\alpha>0$ and every possibly
randomized ERM $A$, an oblivious adversary making at most $\alpha n$ additions can
force $A$ to incur asymptotic expected error at least $\frac12e^{-1/\alpha}$.
\end{theorem}

\formatvspace{-0.6 cm}
\begin{proof}[Proof sketch]
Use the preceding class with $L_k=k$. Consider the ERM learner that returns the target when a 
private label appears, returns $h_{\dollarlabel}$ when all labels are $\dollarlabel$, and otherwise
chooses a fallback at an unseen marker (or any consistent hypothesis, if all markers appear 
in the training set). Let us briefly argue that this learner succeeds in the classic PAC model, 
with i.i.d.\ data. If the true labeling function is $h_{\dollarlabel}$, then learning is trivial.
Otherwise, let $k$ denote the block to which the target function belongs. If $k \leq \epsilon n / 4$,
the claim follows from a union bound over all possible target functions. For larger $k$, it must 
be the case that at least half of markers have mass at most $2 / k$, and an unseen marker must 
exist with high probability. In this case, the ERM learner safely emits the fallback corresponding
to such a point.

For the lower bound, set $k=\lfloor\alpha n\rfloor$, choose a uniform
$k$-subset $A\subseteq C_k$, and take target $c_{k,A}$ with marginal $\Unif(A)$.
An oblivious adversary can then introduce  one $\star$-labeled copy of each marker. Every fallback
then becomes inconsistent with the sample, forcing an ERM to output $c_{k,B}$ for some $k$-subset $B$
containing the distinct observed core points $V$. Conditional on the clean sample,
$A$ is uniform among the $k$-subsets containing $V$. Each of the $k$ points excluded
by $B$ belongs to $A$ with probability $(k-|V|)/(2k-|V|)$. Thus
\[
    \E\,\err_{D_A,c_{k,A}}(c_{k,B})
    \geq\frac12\left(1-\frac1k\right)^n
    \longrightarrow\frac12e^{-1/\alpha}.
\]
Averaging over $A$ fixes a hard target for each $n$, including for randomized ERMs.
\end{proof}

We now observe that Theorem~\ref{thm:erm-fragile} exhibits the strongest possible separation for
ERM learners under monotone corruption. In particular, should \emph{every} ERM learner succeed 
for $\Hcal$ in the classic PAC model, then the same must hold under arbitrary monotone corruption.

\begin{proposition}
\label{prop:all-erm-robust}
If every ERM rule PAC learns $\Hcal$ in the ordinary realizable model, then every ERM rule PAC learns $\Hcal$ under arbitrary adaptive monotone additions.
\end{proposition}
\formatvspace{-0.6 cm}
\begin{proof}
A classical characterization states that every ERM learns $\Hcal$ precisely when its graph dimension is finite \citep{multiclassERM2015}.  Every ERM output on a monotonically augmented sample still interpolates the hidden clean sample.  Graph-dimension uniform convergence controls all such interpolants simultaneously, regardless of how the adversary selected its additions or how the ERM broke ties.
\end{proof}

We conclude the section by noting that in total binary classification, even a single adaptive 
addition can produce a logarithmic increase in ERM's sample complexity.

\begin{proposition}
\label{prop:one-addition-erm}
There is a binary class of VC dimension two such that one adaptive monotone addition can force every possibly randomized ERM to incur expected error $\Omega(\log n/n)$.
\end{proposition}
\formatvspace{-0.6 cm}
\begin{proof}[Proof sketch]
Let the domain be the edges $\binom{\N}{2}$, and let the class contain $h_0\equiv0$
and the incidence hypotheses $h_i(\{u,v\})=\1\{i\in\{u,v\}\}$. For sufficiently large $n$, 
set $N=\lceil8n/\log n\rceil$, choose a uniform target vertex $I\in[N]$,
and draw clean edges uniformly among pairs in $[N]\setminus\{I\}$. Thus every i.i.d.\ label is zero.
Conditional on an i.i.d.\ sample, $I$ is uniform among the vertices $Z\subseteq[N]$ absent from every
observed edge. With high probability, $|Z|\geq2$, in which case one can 
choose $J$ uniformly from $Z\setminus\{I\}$ and insert the datapoint
$(\{I,J\},1)$. Upon insertion, only the hypotheses $h_I,h_J$ remain consistent, and are symmetric
from the learner's perspective. Thus every ERM learner selects the wrong hypothesis with probability $1/2$.
Such an incorrect hypothesis incurs population error $\Theta(1 / N) = \Theta( \log n / n)$, 
completing the argument. The complete argument is deferred to 
\Cref{app:one-addition-erm}.
\end{proof}

\section{Conclusion}\label{sec:conclusion}

Training samples that are drawn i.i.d.~provide two key sources of information: their labels constrain the space of possible target functions, while their unlabeled datapoints suggest which disagreements matter at test time. Roughly speaking, monotone corruptions separate these signals by introducing correctly labeled but otherwise adversarial datapoints.
\citet{og-paper} and \citet{mehrotra2026optimal} demonstrated that total binary classification proves fairly resilient, suffering only a logarithmic deterioration in its optimal rate. Our primary results demonstrate this resilience evaporates beyond binary classification. In particular, we construct learnable problems in multiclass classification and partial binary concepts that become altogether unlearnable under adaptive monotone additions. Oblivious additions, in contrast, remain benign for general learners, though they can dramatically increase the sample complexity of proper learning.

Several questions remain open. First, it remains to establish the precise dependence of monotone error rates on the budget $b$, both in the multiclass and partial binary settings. Second, it remains to find robust replacements to the DS and partial VC dimensions that successfully characterize learnability in the monotone model. In fact, merely establishing useful sufficient conditions for monotone learnability---that are weaker than uniform convergence---would be a notable contribution.

\subsection*{Acknowledgments}
Julian Asilis was supported by the National Science Foundation Graduate Research Fellowship Program under Grant No.\ DGE-1842487,
and by a National Science Foundation award CCF-2239265.
Shaddin Dughmi was supported by NSF Grant CCF-2432219, and completed part of this work while on sabbatical as the Carter and
Tania Neild visiting professor at Northwestern University, as well as a visiting professor in the Data Science Institute at the University of Chicago.
Chirag Pabbaraju was supported by a Google PhD Fellowship, and Moses Charikar’s and Gregory Valiant’s Simons Investigator Awards.

\subsection*{AI disclosure}
ChatGPT 5.6 Pro and Sol Max played a significant role in the development of several results. They were first able to prove \Cref{thm:arbitrary-proper-blowup} after being prompted to modify a poisoned first-Cantor construction designed by the authors, itself based on the first Cantor class of \citet{DS14}. After analyzing this class, the models were prompted to produce counterexamples demonstrating that learnability can be destroyed in the multiclass and/or partial binary settings. After several failed attempts and limited feedback from the authors (e.g., to avoid constructions that were becoming overly convoluted), the models were eventually able to discover the counterexamples underlying \Cref{thm:multiclass-destruction,thm:partial-destruction}. These counterexamples were originally even more intricate than the current constructions, and simplified (somewhat) by ChatGPT and the human authors.
By contrast, the positive results of \Cref{sec:harmless} were developed independently by the authors.
These systems were additionally used to draft and edit portions of the manuscript.
The authors take full responsibility for all content in the paper.

\clearpage
\bibliographystyle{plainnat}
\IfFileExists{refs.bib}{
  \bibliography{refs}
}{
  \bibliography{../refs}
}

\clearpage
\begin{appendices}
\appendix
\crefalias{section}{appendix}
\crefalias{subsection}{appendix}
\section{Proofs for Harmless Regimes}\label{app:tools}

This appendix supplies the omitted proofs for the three positive mechanisms of
\Cref{sec:harmless}. Exchangeable reserves preserve the leave-one-out symmetry
exploited by one-inclusion prediction; random thinning attempts to recover a
clean subsample explicitly; and partition-based aggregation confines the
adversarial insertions to a bounded number of groups. We treat these arguments
in the same order as the main text.

\subsection{Exchangeable reserves}\label{app:exchangeable-reserve}

\begin{proofof}{\Cref{thm:exchangeable-reserve}}
Fix the additional examples and all clean information used to choose them,
and then fix the complete realization of $Z_0,Z_1,\ldots,Z_r$. For each
$i\in\{0,\ldots,r\}$, withhold $Z_i$ and apply the canonical one-inclusion
rule to the union of the fixed additional examples and
$\{Z_j:j\neq i\}$. These $r+1$ predictions belong to one common
one-inclusion problem: only the held-out coordinate changes.
By \Cref{thm:oig-density}, at most $d$ of these leave-one-out predictions are
incorrect. Exchangeability makes every coordinate equally likely to be the
fresh example $Z_0$, so its expected error is at most $d/(r+1)$.
\end{proofof}

\begin{proofof}{\Cref{cor:oblivious-free}}
Condition on the complete multiset of oblivious additions. The $n$ clean
training examples and a fresh test example remain exchangeable, so
\Cref{thm:exchangeable-reserve} applies with $r=n$.
\end{proofof}

\begin{proofof}{\Cref{cor:semi-adaptive}}
Condition on the indices and values of the $m$ clean examples revealed to the
adversary, together with every addition chosen from them. The remaining $n-m$
clean examples are still i.i.d.\ from $D$ and exchangeable with a fresh test
point. Apply \Cref{thm:exchangeable-reserve} with $r=n-m$.
\end{proofof}

We note that the same argument also covers a random reserve. Suppose each clean example is
revealed independently with probability $p<1$, and let
$R\sim\Bin(n,1-p)$ be the number left hidden. Conditional on $R$,
\Cref{thm:exchangeable-reserve} gives error at most $d/(R+1)$. Therefore
\[
    \E \, L_D
    \le d \, \E \,\frac1{R+1}
    =d\int_0^1\bigl(p+(1-p)t\bigr)^n\,dt
    =d\frac{1-p^{n+1}}{(1-p)(n+1)}.
\]

\subsection{Random thinning}\label{app:thinning}

\begin{proofof}{\Cref{thm:thinning}}
Suppose first that the adversary inserts exactly $b'$ examples, where
$b'\le b$. The $t$ selected positions are all clean with probability
$\binom{n}{t}/\binom{n+b'}{t}$. Conditional on this event, the selected
examples correspond to a uniformly random set of $t$ distinct clean indices.
Since the clean observations are i.i.d., the selected sample has law $D^t$,
and the conditional risk is at most $R_{\clean}(A_t,t)$. On the complementary
event we bound the loss by one. Thus
\[
    R_{\ad,b}(n)
    \le R_{\clean}(A_t,t)
       +1-\frac{\binom nt}{\binom{n+b'}t}
    \le R_{\clean}(A_t,t)
       +1-\frac{\binom nt}{\binom{n+b}t}.
\]

For the simpler estimate, every selected position is inserted with marginal
probability $b'/(n+b')$. A union bound over the $t$ selected positions gives
contamination probability at most
$b't/(n+b')\le bt/(n+b)$. The learner only invokes $A_t$, so properness,
impropriety, or any other restriction on the output class is inherited
unchanged.
\end{proofof}

\begin{proofof}{\Cref{cor:universal-threshold}}
Let $A_t$ be ordinary learners whose clean risk tends to zero. Since
$b_n/n\to0$, one may choose $t_n\to\infty$ sufficiently slowly that both
$R_{\clean}(A_{t_n},t_n)$ and $b_nt_n/(n+b_n)$ vanish. Applying
\Cref{thm:thinning} proves ordinary learnability under $b_n=o(n)$ additions.
If each $A_t$ is proper, then so is the thinned learner. The converse is
exactly \Cref{thm:multiclass-destruction}.
\end{proofof}

\subsection{Partition, plurality, and list robustification}
\label{app:plurality}

\begin{proofof}{\Cref{thm:plurality}}
Suppose the adversary inserts $b'\le b$ examples. As a preprocessing step,
duplicate arbitrary observed examples until the declared final length $n+b$
is reached; treating these $b-b'$ duplicates as side positions only
strengthens the adversary. Put $r=4b+1$ and
$t=\lfloor(n+b)/r\rfloor$, choose $rt$ positions uniformly if necessary, and
partition them uniformly into $r$ groups of size $t$. At most $b$ groups
contain a side position, so at least $r-b$ groups are clean.

Fix a test point $x$, and let $w(x)$ be the number of clean-group predictors
that err at $x$. If the plurality is wrong, some incorrect label receives at
least as many votes as the correct label. Even granting all $b$ dirty groups
to this incorrect label, we must have
$w(x)+b\ge(r-b)-w(x)$, or equivalently
$w(x)\ge(r-2b)/2$. Hence
\[
    \mathbf1\{\mathrm{plurality}(x)\neq h^\star(x)\}
    \le \frac{2}{r-2b}
       \sum_{j:G_j\text{ clean}}
       \mathbf1\{A_t(G_j)(x)\neq h^\star(x)\}.
\]
Each clean group has law $D^t$. Taking expectations and then enlarging the
sum to all $r$ groups yields
\[
    R^{\imp}_{\ad,b}(n;\H)
    \le \frac{2r}{r-2b} \, R^{\imp}_{\clean}(t;\H).
\]
For $r=4b+1$, the prefactor is strictly smaller than four. If
$d_{\DS}(\H)=d$, the ordinary one-inclusion rate
$R^{\imp}_{\clean}(t;\H)=O(d/(t+1))$ gives the claimed bound
$O(d(b+1)/(n+b))$.
\end{proofof}

\begin{theorem}\label{thm:list-robustification}
Let $s\ge1$, put $r=b+s$ and $t=\lfloor(n+b)/r\rfloor$, and perform the same
padding and random partition. Run an ordinary learner independently on each
group and output the pointwise list of its $r$ predictions. If the base
learner has population risk at most $\epsilon$ with probability at least
$1/2$ on a clean group, then the resulting list has risk greater than
$\epsilon$ with probability at most $2^{-s}$. Consequently, if
$d_{\DS}(\H)=d$, taking $s=\lceil\log_2(1/\delta)\rceil$ gives list size
$b+O(\log(1/\delta))$ using
$O((d/\epsilon)(b+\log(1/\delta)))$ clean examples.
\end{theorem}

\begin{proof}
At most $b$ groups contain a side position, so at least $s$ groups are clean.
Conditional on the partition and on the identities of these clean groups,
they consist of disjoint collections of independent clean observations.
Using independent learner randomness across groups, their output predictors
are therefore independent. Choose any $s$ clean groups. The probability that
every one of their predictors has risk greater than $\epsilon$ is at most
$2^{-s}$. If even one has risk at most $\epsilon$, then the pointwise list,
which contains that predictor's label at every test point, also has risk at
most $\epsilon$.

When $d_{\DS}(\H)=d$, an ordinary one-inclusion learner has expected risk
$O(d/(t+1))$. Taking $t=O(d/\epsilon)$ and applying Markov's inequality gives
the required constant-probability guarantee on each clean group. Multiplying
this group size by $r=b+s$ yields the stated sample complexity.
\end{proof}
\section{Multiclass Lower Bound}
\label{sec:multiclass-lower-bound}

In this section, we will restate and prove \Cref{thm:multiclass-destruction}. Namely, we will construct a multiclass hypothesis class that is properly learnable with clean i.i.d.\ data, but is not learnable at all under monotone adversarial corruptions.

\MulticlassLowerBound*

We will establish the proof of theorem step-by-step in the following subsections.

\subsection{The Hard Class $\Hcal$}
\label{sec:hard-class}

We begin by concretely describing the multiclass hypothesis that witnesses the lower bound in \Cref{thm:multiclass-destruction}.

\paragraph{Preliminaries.} For any even integer $r \ge 2$, let $\Bcal_r$ denote the family of all subsets of $[r] = \{1, \ldots, r \}$ that have size $r/2$, and let $K_r$ be the set of all binary strings of length $|\Bcal_r|=\binom{r}{r/2}$. Namely,
\begin{align}
    \label{eqn:def-B-r-K-r}
    \Bcal_r := \binom{[r]}{r/2}, \qquad K_r := \{0,1\}^{\Bcal_r}.
\end{align}
We index the coordinates of any member $k \in K_r$ by halfsets $B \in \Bcal_r$, i.e., $k=(k_B)_{B \in \Bcal_r}$. For a halfset $B \in \Bcal_r$, let $e_B$ be the indicator vector for $B$. Then, $k \oplus e_B$ is the vector which flips the bit at index $B$ in $k$. For $u, v \in K_r$, let $d(u,v)$ denote the Hamming distance between $u$ and $v$, i.e.,
\begin{align}
    d(u,v) := |\{B \in \Bcal_r: u_B \neq v_B\}|.
\end{align}
For $t \in [r]$, we define $a \in \{0,1,*\}^{\Bcal_r}$ to be a \textit{$t$-masked vector}, when $a$ has value $*$ precisely at the coordinates $B$ satisfying $t \in B$. Concretely,
\begin{align}
    a_B = * \iff t \in B.
\end{align}

\paragraph{The Class.}
Define $\Xcal_r$ --- the domain of the hypothesis class at level $r$ --- to be the set of all $t$-masked vectors over halfsets of $[r]$:
\begin{align}
    \label{eqn:def-X-r}
    \Xcal_r := \{(r,t,a): t \in [r], a \text{ is a $t$-masked vector}\}.
\end{align}
We include the tags $r,t$ deliberately as part of the tuple. The complete domain $\Xcal$ for the class is then defined to be
\begin{align}
    \label{eqn:def-X}
    \Xcal := \bigsqcup_{\substack{r \ge 2 \\ r \text{  even }}} \Xcal_r.
\end{align}
The label space $\Ycal$ for the class comprises of all valid tuples $(r,k)$:
\begin{align}
    \Ycal := \{(r,k): r \ge 2,\, r\text{ even},\, k \in K_r\}. 
\end{align}
The hypothesis class $\Hcal$ will have a hypothesis $h_{r,k}$ for every valid tuple $(r,k)$. We first specify the action of $h_{r,k}$ on tuples in $\Xcal_s$, where $s \neq r$:
\begin{align}
    h_{r,k}(x) := (r,k), \qquad \text{if $x \in \Xcal_s$ where $s \neq r$}.
\end{align}
That is, $h_{r,k}$ simply maps any $x$ from a different level to $(r,k)$. In order to specify the action of $h_{r,k}$ on members of its own level $\Xcal_r$, we define the notion of \textit{visible discrepancy} of any $k \in K_r$ with a $t$-masked vector $a$:
\begin{align}
    \Delta_a(k) := \{B \in \Bcal_r: a_B \neq * \text{ and } a_B \neq k_B\}.
\end{align}
In words, $\Delta_a(k)$ comprises of all the halfsets $B$ at which $a_B$ is not $*$, and $k_B$ disagrees with $a_B$.

Then, for any $x=(r,t,a) \in \Xcal_r$, we define
\begin{align}
    h_{r,k}(x) := \begin{cases}
        (r, k \oplus e_B) & \text{if $\Delta_a(k)=\{B\}$,} \\
        (r,k) & \text{otherwise}.
    \end{cases}
\end{align}
To summarize, a hypothesis $h_{r,k}$ usually labels an $x \in \Xcal$ with its own name $(r,k)$. It deviates from this behavior, \textit{only when} $x=(r,t,a)$ belongs to its own level $\Xcal_r$, \textit{and} the visible discrepancy of $k$ with $a$ is precisely at a \textit{single} halfset $B$, in which case $h_{r,k}$ labels $x$ with $k$ flipped at the index $B$. Thus, if $h_{r,k}$ assigns the label $y$ to any $x$, it holds that $d(k,y) \le 1$.

In what follows, for any $k \in K_r$ and $t \in [r]$, we will let $x_t(k) \in \Xcal_r$ denote the instance whose $t$-masked vector is constructed from $k$, i.e.,
\begin{align*}
    x_t(k) = (r,t,a), \quad \text{where } a_B = \begin{cases}
        * & \text{if $t \in B$} \\
        k_B & \text{otherwise}.
    \end{cases}
\end{align*} 

\subsection{$\Hcal$ is Properly Learnable with Clean Data}
\label{sec:hard-class-properly-learnable-with-clean-data}

We now construct a proper learner $\Lcal$ for the class $\Hcal$ constructed above. Given a clean i.i.d.\ sample $C=\{(x_i,y_i)\}_{i \in [n]}$, the learner $\Lcal$ constructs a candidate hypothesis list $\Gamma$ comprising of hypotheses from $\Hcal$. This list has a slot for every $i \in [n]$, and two slots for every pair $i < j$; thus, the size of $\Gamma$ is $n + 2\binom{n}{2}=n^2$. We may think of the slots as being ordered by putting all the singleton $i$'s first (in increasing order), followed by all the pairs $(i,j)$ (in lexicographic order).

The learner populates each slot with either a candidate hypothesis from $\Hcal$, or leaves it blank.
\begin{enumerate}
    \item For every $i \in [n]$, consider the labeled sample $(x_i, y_i)$, and let $y_i = (r_i, k_i)$. The learner populates slot $i$ in $\Gamma$ with the hypothesis $h_{r_i, k_i}$.
    \item For every pair $(i,j)$ where $i < j$, if $r_i \neq r_j$ or $k_i=k_j$, the learner leaves both the slots corresponding to the pair $(i,j)$ blank. Otherwise, we have that $r_i=r_j=r$ and $k_i \neq k_j$. In this case, consider the set
    \begin{align}
        \label{eqn:small-hamming-distance-from-both}
        I := \{u: d(u,k_i) \le 1 \text{ and } d(u, k_j) \le 1\}.
    \end{align}
    We have that $I$ is either empty, or has exactly two members. This is because, if $I$ is non-empty, then by the triangle inequality, $d(k_i, k_j) \le 2$. If $d(k_i, k_j)=1$, then $I=\{k_i, k_j\}$; otherwise, $d(k_i, k_j)=2$, in which case $I$ comprises of the two intermediate binary vectors $u$ and $v$ that are each at a distance of $1$ from $k_i$ and $k_j$. 

    If $I$ is empty, the learner leaves the two slots corresponding to $(i,j)$ blank. Otherwise, it populates the two slots with $h_{r,u}$ and $h_{r,v}$ corresponding to the two members $u$ and $v$ of $I$ (ordered according to a fixed background ordering of $K_r$).
\end{enumerate}
After having populated the list $\Gamma$ as above, the learner $\Lcal$ returns the first candidate in $\Gamma$ that is consistent with the entire sample $C$; if no candidate in $\Gamma$ is consistent with $C$ or $\Gamma$ is empty, then the leaner returns a default member of $\Hcal$. Thus, $\Lcal$ is a deterministic, proper learner.

We now argue that if $C=\{(x_i,y_i)\}_{i \in [n]}$ is realizable by some $h_{r,k}$, then there is always a consistent candidate in the list $\Gamma$ constructed by the learner $\Lcal$. 

To see this, we first observe that for any $x=(r,t,a)$, if $h_{r,k}(x)=(r,k')$, then $d(k,k') \le 1$ and $h_{r,k'}(x)=(r,k')$. This holds because either (1) $k'=k$, in which case the claim is immediate, or (2) $k'=k \oplus B$, where $\Delta_a(k)=\{B\}$. In this case, $d(k,k')=1$, and $h_{r,k'}(x)=(r,k')$, since $\Delta_a(k')=\emptyset$.

So then, if $C=\{(x_i,y_i)\}_{i \in [n]}$ is realizable by $h_{r,k}$, every $y_i=(r_i, k_i)$ must satisfy that $r_i=r$, and also that $d(k, k_i) \le 1$. 

Now, if every $k_i=k'$ for some $k'$, meaning that
\begin{align*}
    h_{r,k}(x_i)=(r,k') \quad \text{for every $i$},
\end{align*}
then by the above reasoning, $h_{r,k'}(x_i)=(r, k')$ for every $i$, and thus, the candidate $h_{r,k'}$ (which exists at every slot $i$) is consistent with all of $C$.

Otherwise, there exists some $i < j$ for which $k_i \neq k_j$. In this case, since both $d(k, k_i) \le 1$ and $d(k, k_j) \le 1$, we have that $k$ belongs to the set $I$ constructed in \eqref{eqn:small-hamming-distance-from-both} for the pair $(i,j)$. Thus, $h_{r,k}$ itself belongs to one of the two slots corresponding to this pair.

We now proceed to argue that if $C=\{(x_i,y_i)\}_{i \in [n]}$, where $x_1,\dots,x_n$ are drawn i.i.d.\ from $\Dcal$ and $y_i=h_{r,k}(x_i)$ for all $i$, then with good probability, $\Lcal(C)$ has small error.

Towards this, let $E_i$ be the event that the slot $i$ in the list $\Gamma$ populated by the learner has a candidate hypothesis that has error larger than $\epsilon$, but is consistent with all of $C$. Observe that since the candidate hypothesis at slot $i$ is determined solely by $(x_i,y_i)$, and the samples $(x_j, y_j)$ for $j \neq i$ are independent of $(x_i, y_i)$, we have that conditioned on slot $i$ containing a hypothesis having error larger than $\epsilon$, the probability that this hypothesis is consistent with $C$ is at most $(1-\epsilon)^{n-1}$. In other words, for every $i$,
\begin{align*}
    \Pr[E_i] \le (1-\epsilon)^{n-1}.
\end{align*}
Similarly, let $E_{i,j}$ be the event that the two slots corresponding to $(i,j)$ in the list $\Gamma$ populated by the learner have a candidate hypothesis that has error larger than $\epsilon$, but is consistent with all of $C$. Again, the candidate hypotheses at these slots are determined solely by $(x_i,y_i)$ and $(x_j, y_j)$, and the samples $(x_\ell, y_\ell)$ for $\ell \notin \{i,j\}$ are independent of $(x_i, y_i)$ and $(x_j, y_j)$. Therefore, conditioned on the slots corresponding to $(i,j)$ containing a hypothesis having error larger than $\epsilon$, the probability that this hypothesis is consistent with $C$ is at most $(1-\epsilon)^{n-2}$. In other words, for every $(i,j)$,
\begin{align*}
    \Pr[E_{i,j}] \le (1-\epsilon)^{n-2}.
\end{align*}
We thus have that
\begin{align*}
    \Pr[\err_{\Dcal, h_{r,k}}(\Lcal(C)) > \epsilon] &\le \Pr\left[\bigcup_i E_i \cup \bigcup_{i < j} E_{i,j}\right]
    \le \sum_{i} \Pr[E_i] + \sum_{i < j}\Pr[E_{i,j}] \tag{union bound}\\
    &\le n(1-\epsilon)^{n-1} + \binom{n}{2}(1-\epsilon)^{n-2} \le n^2e^{-\epsilon(n-2)} \tag{since $1+x\le e^{x}$ for all real $x$} \\
    &\le \delta,
\end{align*}
where in the last inequality, we used that $n \ge 100\cdot\frac{\log(1/\epsilon)+\log(1/\delta)}{\epsilon}$. 

While the analysis above gives a proper learner for the class $\Hcal$, its sample complexity is suboptimal by a logarithmic factor. However, it can be shown that the DS dimension of this class is only a constant, which implies, by the result of \cite{pabbaraju2026optimal}, that this class is improperly learnable with the optimal sample complexity of $O(\log(1/\delta)/\epsilon)$.


\begin{proposition}
    \label{prop:masked-vector-ds}
    The class $\Hcal$ has DS dimension exactly $2$.
\end{proposition}
\formatvspace{-0.6 cm}
\begin{proof}

For the lower bound, we will work in the $r=4$ level. Consider the halfsets $A=\{1,3\}$, $B=\{2,3\}$, and let $0\in K_4$ be the all-zero vector. 
Consider
$x_1=x_1(0)$ and $x_2=x_2(0)$, and for $a,b\in\{0,1\}$ define
$k_{a,b}=ae_A\oplus be_B$, where $e_A, e_B$ are singleton bit vectors at $A$ and $B$ respectively. At $x_1$, the coordinate $A$ is masked, whereas the coordinate $B$ is not; so $h_{4,k_{a,b}}(x_1)=(4,ae_A)$. Symmetrically,
$h_{4,k_{a,b}}(x_2)=(4,be_B)$. The four restrictions of $a,b$ therefore form a
$2\times2$ grid at $x_1,x_2$, and hence a two-dimensional pseudo-cube. Therefore, the DS dimension of $\Hcal$ is at least $2$.

Now towards the upper bound: fix any $r$ and $u \in K_r$. Define
\begin{align*}
    \Fcal_u := \{h_{r,k}: d(k,u) \le 1\}.
\end{align*}
That is, $\Fcal_u$ consists of all the hypotheses $h_{r,k}$ where $k$ is within the Hamming ball of radius 1 around $u$. The key observation is that any given point $x$, there can be at most one label shared by two distinct members of $\Fcal_u$. 

To see this, observe first that if $x \notin \Xcal_r$, then every hypothesis in $\Fcal_u$ assigns its own name to $x$, and hence no two distinct hypotheses can share a label. Otherwise, let $x=(r,t,a)$ for some $t$-masked vector $a$, and suppose that a label $(r,v)$ is shared at $x$ by two distinct members $h_{r,k}$ and $h_{r,\ell}$ of $\Fcal_u$, where $k \neq \ell$.

We now claim that $v$ agrees with $a$ wherever $a$ is non-$*$. The only way that a hypothesis does not label a point by its own name is if there is exactly one visible discrepancy with the $t$-masked vector, and hence, at least one of $k$ or $\ell$ disagree with $a$ at exactly one visible coordinate; furthermore, after this visible discrepancy is fixed, the resulting vector --- which is $v$ --- agrees with $a$ at all visible coordinates. Therefore, we have that $v$ agrees with $a$ wherever it is non-$*$. 

Next, we claim that $v$ agrees with $u$ wherever $a$ is $*$. To see this, fix a masked coordinate $B$ which satisfies that $a_B=*$. Since the labeling rule never changes a masked coordinate, we have that $k_B=\ell_B=v_B$. Now suppose that $v_B \neq u_B$. Then, both $k$ and $\ell$ differ from $u$ at coordinate $B$. Since $d(k,u) \le 1$ and $d(\ell, u) \le 1$, this must be the only coordinate at which $k$ and $\ell$ disagree with $u$, meaning that $k=\ell=u \oplus e_B$, which contradicts that $k$ and $\ell$ are distinct. Thus, we have that $v_B=u_B$ at all masked coordinates $B$.

Summarily, we have precisely pinned down the form of a label $(r,v)$ that is shared by two distinct hypotheses from $\Fcal_u$ at $x=(r,t,a)$: at visible coordinates, $v$ agrees with $a$, and at masked coordinates, $v$ agrees with $u$.

Using this property, we will now argue that the DS dimension is at most 2. Suppose for the sake of contradiction that some set $\{x_1,x_2,x_3\}$ of three distinct points is DS-shattered by $\Hcal$. This means that there exists a finite, non-empty collection $P$ of label triples realizable by $\Hcal$, with the property that for each triple in $P$, and for each of the three points, we can find another triple in $P$ that disagrees with this triple only at that point.

So, choose a triple $(U,A,B) \in P$, and say $U=(r,u)$. For the rest of the argument, we will only consider triples in $P$ that realize $(r,u)$ at $x_1$. Recall that by definition of the labeling rule, any hypothesis $h_{r,k}$ that realizes a label $(r,u)$ must satisfy that $d(k,u) \le 1$, and therefore, any hypothesis that realizes a triple in $P$ which fixes the label $(r,u)$ at $x_1$ belongs to $\Fcal_u$.

Now, by the DS property, there exists a triple $(U,A, C) \in P$, where $B \neq C$. Consider the hypotheses that realize the triplets $(U,A,B)$ and $(U,A,C)$ --- these hypotheses are distinct, and by our reasoning above, both of them belong to $\Fcal_u$. Consequently, $A$ is a label that is shared at $x_2$ by two distinct hypotheses in $\Fcal_u$.

But the DS property also guarantees the existence of $(U,D,B) \in P$, where $A \neq D$, which furthermore requires the existence of $(U,D,E) \in P$, where $E \neq B$. So then, consider the hypotheses that realize the triplets $(U,D,B)$ and $(U,D,E)$ --- these hypotheses are again distinct, and by our reasoning above, both of them belong to $\Fcal_u$. Consequently, $D \neq A$ is also a label that is shared at $x_2$ by two distinct hypotheses in $\Fcal_u$. This contradicts the property above that at most one label can be shared by two distinct hypotheses in $\Fcal_u$ at a single point. Thus, $\{x_1,x_2,x_3\}$ cannot be DS-shattered by $\Hcal$.
\end{proof}


\subsection{The Two Indistinguishable Worlds}
\label{sec:indistinguishable-worlds}

We now proceed towards systematically establishing how this properly learnable class may be rendered unlearnable by a monotone adversary. 
Fix any $n$, and in what follows, fix 
\begin{align*}
    r := 2n.
\end{align*}
Recall that for any $k \in K_r$ and $t \in [r]$, $x_t(k) \in \Xcal_r$ denotes the instance whose $t$-masked vector is constructed from $k$, i.e.,
\begin{align*}
    x_t(k) = (r,t,a), \quad \text{where } a_B = \begin{cases}
        * & \text{if $t \in B$} \\
        k_B & \text{otherwise}.
    \end{cases}
\end{align*}
Since we include $t$ in the tuple corresponding to $x_t(k)$, we have that $x_1(k),\dots,x_r(k)$ are all distinct. Now, let $D_k$ be the uniform distribution over $\{x_1(k),\dots,x_r(k)\}$. Observe that since we derive the $t$-masked vector in every $x_t(k)$ from $k$ itself, it holds that 
\begin{align}
    \label{eqn:hrk-xtk}
    h_{r,k}(x_t(k))=(r,k). 
\end{align}

We now claim that both $h_{r,k}$ and $h_{r, k \oplus e_B}$ disagree on points $x_t(k)$ and $x_t(k \oplus e_B)$ when $t \in B$, but agree on them when $t \notin B$.

\begin{claim}
    \label{claim:same-labels-different-hypotheses}
    Let $k \in K_r$; fix any $B \in \Bcal_r$, and let $w = k \oplus e_B$. Then, for any $t \in [r]$:
    \begin{enumerate}
        \item[(1)] If $t \in B$, then $x_t(k) = x_t(w)$. Furthermore,
        \begin{align*}
            &h_{r,k}(x_t(k)) = (r,k),
            &h_{r,w}(x_t(w)) = (r,w).
        \end{align*}
        \item[(2)] If $t \notin B$, then $x_t(k) \neq x_t(w)$. Furthermore,
        \begin{align*}
            &h_{r,k}(x_t(k)) = h_{r,w}(x_t(k)) = (r,k),
            &h_{r,k}(x_t(w)) = h_{r,w}(x_t(w)) = (r,w).
        \end{align*}
    \end{enumerate}

\end{claim}
\begin{proof}
    Let $x_t(k)=(r,t,a)$ and $x_t(w)=(r,t,a')$. 
    
    For (1): If $t \in B$, observe that $a=a'$ (since $w$ and $k$ only disagree at index $B$, which gets suppressed to $*$ in both $a$ and $a'$), and hence $x_t(k)=x_t(w)$. The latter two equalities follow from \eqref{eqn:hrk-xtk}.

    For (2): By definition, $k_B \neq w_B$, and since $t \notin B$, we have that $a_B \neq a'_B$, and hence, $x_t(k) \neq x_t(w)$.
    
    For the latter equalities, from \eqref{eqn:hrk-xtk}, we have $h_{r,k}(x_t(k))=(r,k)$. Since $w = k \oplus e_B$ and $t \notin B$, the $t$-masked vector $a$ in $x_t(k)$ satisfies that $\Delta_a(w)=\{B\}$; thus, $h_{r,w}(x_t(k))=(r, w \oplus e_B)=(r,k)$ as desired. The last equality in the claim follows by interchanging $k$ and $w$ in this reasoning. 
\end{proof}

Using \Cref{claim:same-labels-different-hypotheses}, we now claim that no predictor can minimize error simultaneously in the two worlds $(D_k, h_{r,k})$ and $(D_w, h_{r,w})$.

\begin{claim}
    \label{claim:two-world-confusion}
    Let $k \in K_r$; fix any $B \in \Bcal_r$, and let $w = k \oplus e_B$. For any deterministic total predictor $g:\Xcal \to \Ycal$,
    \begin{align*}
        \err_{D_k, h_{r,k}}(g) + \err_{D_w, h_{r,w}}(g) \ge \frac{1}{2}.
    \end{align*}
\end{claim}
\begin{proof}
    We have that
    \begingroup
    \allowdisplaybreaks
    \begin{align*}
        \err_{D_k, h_{r,k}}(g) + \err_{D_w, h_{r,w}}(g) &\ge \frac{1}{r} \sum_{t \in B} \left(\Ind[g(x_t(k)) \neq h_{r,k}(x_t(k))] + \Ind[g(x_t(w)) \neq h_{r,w}(x_t(w))]\right) \\
        &= \frac{1}{r} \sum_{t \in B} \left(\Ind[g(x_t(k)) \neq h_{r,k}(x_t(k))] + \Ind[g(x_t(k)) \neq h_{r,w}(x_t(w))]\right) \tag{since $x_t(k)=x_t(w)$ when $t \in B$ from \Cref{claim:same-labels-different-hypotheses}, part (1)}\\
        &\ge \frac{1}{r} \sum_{t \in B} 1 \tag{since $h_{r,k}(x_t(k)) \neq h_{r,w}(x_t(w))$ when $t \in B$ from \Cref{claim:same-labels-different-hypotheses}, part (1)} \\
        &= \frac{1}{2}. \tag{since $|B|=r/2$}
    \end{align*}
    \endgroup
\end{proof}

\subsection{The Monotone Adversary's Goal}
\label{sec:monotone-adversary-action}

We will now construct a monotone adversary that exploits \Cref{claim:two-world-confusion}. In particular, the adversary will aim to append carefully chosen samples to the clean dataset, so that the combined set of samples (post shuffling) confuses the learner whether the underlying true distribution is $(D_k, h_{r,k})$ or $(D_w, h_{r,w})$.

We first pin down the form of the final enlarged dataset that the adversary will target, and which causes this confusion for the learner. The strategy is as follows: the adversary will aim for the final dataset to have a copy of $(x_t(w), (r,w))$ to match every $(x_t(k), (r,k))$ in it. Then, a learner looking at this dataset does not know whether the clean data distribution was $D_k, h_{r,k}$ and the adversary added clones from  $D_w, h_{r,w}$, or vice-versa

More concretely, let $\Fcal_n$ comprise of all ``frequency-vectors'' of $n$ points drawn from a domain of $r$ points:
\begin{align*}
    \Fcal_n := \left\{f=(f_1,\dots,f_r) \in \Z_{\ge 0}^r: \sum_{t=1}^r f_t = n\right\}.
\end{align*}
For $f \in \Fcal_n$, let $\supp(f)=\{t:f_t > 0\}$. Since $\sum_{t=1}^r f_t = n$, and $r=2n$, there are at least $n$ entries in $f$ that are zero. Fix an ordering over $\Bcal_r$, and let
\begin{align*}
    \beta(f) := \min \{B \in \Bcal_r: B \cap \supp(f) = \emptyset\}.
\end{align*}
By our previous reasoning, $\beta(f)$ is well-defined for every $f \in \Fcal_n$.

Now, for any $f \in \Fcal_n$, $k \in K_r$, $B=\beta(f)$ and $w = k \oplus e_B$, define the labeled (unordered) multiset
\begin{align}
    T_{k,f} := \biguplus_{t \in [r]} 
    f_t \cdot \left\{\left(x_t(k), (r,k)\right), \left(x_t(w), (r,w)\right)\right\}.
\end{align}
Here, the notation $f_t \cdot \left\{\left(x_t(k), (r,k)\right), \left(x_t(w), (r,w)\right)\right\}$ denotes a multiset having $f_{t}$ copies of the pair of labeled points$\left(x_t(k), (r,k)\right), \left(x_t(w), (r,w)\right)$.
By definition of $B=\beta(f)$, whenever $f_t > 0$, it holds that $t \notin B$. Therefore, from part (2) in \Cref{claim:same-labels-different-hypotheses}, we have that $T_{k,f}$ is consistent with $h_{r,k}$. This will be crucially exploited by the adversary later.

Note that $|T_{k,f}|=2n$. More importantly, we can also see that $T_{k,f}=T_{w,f}$ (as unordered multisets), and hence, if $\Pi$ is a uniformly random permutation, it holds that $\Pi(T_{k,f})$ and $\Pi(T_{w,f})$ have the same distribution.

So, let $\Lcal$ be any learner, possibly using randomness $\sigma$. We have that
\begin{align}
    &\E_{\Pi, \sigma} \left[\err_{D_k, h_{r,k}}(\Lcal(\Pi(T_{k,f})))\right] + \E_{\Pi, \sigma} \left[\err_{D_w, h_{r,w}}(\Lcal(\Pi(T_{w,f})))\right] \nonumber \\
    &= \E_{\sigma} \left[\E_{\Pi}\left[\err_{D_k, h_{r,k}}(\Lcal(\Pi(T_{k,f})))\right] + \E_{\Pi} \left[\err_{D_w, h_{r,w}}(\Lcal(\Pi(T_{w,f})))\right]\right] \tag{$\sigma$ is independent of $\Pi$} \nonumber \\
    &= \E_{\sigma} \left[\E_{\Pi}\left[\err_{D_k, h_{r,k}}(\Lcal(\Pi(T_{k,f})))\right] + \E_{\Pi} \left[\err_{D_w, h_{r,w}}(\Lcal(\Pi(T_{k,f})))\right]\right] \tag{$\Pi(T_{w,f})$ and $\Pi(T_{k,f})$ are distributionally identical} \nonumber \\
    &= \E_{\sigma} \E_{\Pi}\left[\err_{D_k, h_{r,k}}(\Lcal(\Pi(T_{k,f}))) + \err_{D_w, h_{r,w}}(\Lcal(\Pi(T_{k,f})))\right]  \\
    &\ge \E_{\sigma} \E_{\Pi}\left[\frac12\right] = \frac12, \label{eqn:multiclass-two-world-expected-risk}
\end{align}
where in the last inequality, we used that, conditioned on $\Pi$ and $\sigma$, $\Lcal(\Pi(T_{k,f})$ is a deterministic total predictor, together with \Cref{claim:two-world-confusion}. 

We now proceed to choose $k$ (which determines the two confusing target hypotheses $h_{r,k}$ and $h_{r,w}$ --- note that the choice of $k$ can depend on the learner and $n$, but should not depend on the realized sample/count vector $f$.

The following proposition provides us with the required choice of $k$:

\begin{proposition}[Choice of $k$]
    \label{prop:k-choice}
    Let $f$ be the frequency-vector of $n$ uniformly random draws from $[r]$. There exists $k^\star \in K_r$, which depends only on $n$ and $\Lcal$, satisfying
    \begin{align}
        \E_{f}\left[E_{\Pi, \sigma} \left[\err_{D_{k^\star}, h_{r,k^\star}}(\Lcal(\Pi(T_{k^\star,f})))\right]\right] \ge \frac14.
    \end{align}
\end{proposition}
\begin{proof}
    For convenience, define
    \begin{align*}
        L_f(k) := E_{\Pi, \sigma} \left[\err_{D_k, h_{r,k}}(\Lcal(\Pi(T_{k,f})))\right]
    \end{align*}
    for $f \in \Fcal_n$. Equation \eqref{eqn:multiclass-two-world-expected-risk} gives us that for every $k \in K_r$,
    \begin{align*}
        L_f(k) + L_f(k \oplus e_{\beta(f)}) \ge \frac12.
    \end{align*}
    Now observe that as we range $k$ over all the vectors in $K_r$, the vector $k \oplus e_{\beta(f)}$ also ranges over all the vectors in this set. Thus, if we sum the inequality above over every $k \in K_r$, we count each $k \in K_r$ twice on the left-hand side, giving us that
    \begin{align*}
        \frac{1}{|K_r|} \sum_{k \in K_r} L_f(k) \ge \frac14.
    \end{align*}
    We therefore have that
    \begin{align*}
        \frac{1}{|K_r|} \sum_{k \in K_r} \E_f L_f(k) = \E_f \left[\frac{1}{|K_r|} \sum_{k \in K_r}  L_f(k) \right] \ge \frac14.
    \end{align*}
    In particular, this implies that there exists $k^\star \in K_r$ for which $\E_f L_f(k^\star) \ge \frac14$, as desired.
\end{proof}

\subsection{Concluding the Lower Bound}
\label{sec:lower-bound-conclusion}

We are now ready to fully specify the lower bound instance and the monotone adversary $\Acal$. Let $k^\star \in K_r$ be supplied by \Cref{prop:k-choice}, solely as a function of $\Lcal$ and $n$. Fix the target hypothesis and data distribution to be
\begin{align*}
    h^\star := h_{r, k^\star}, \qquad D := D_{k^\star}.
\end{align*}

Let $C=\{(x_i,y_i)\}_{i \in [n]}$ be a clean sample, where $x_1,\dots,x_n$ are i.i.d.\ draws from $D$, and $y_i = h^\star(x_i)$. The adversary $\Acal$ obtains the frequency-vector $f=(f_t)_{t \in [r]}$ for each $x_t(k^\star)$. Note that $f$ is distributed as the frequency-vector of $n$ uniformly random draws from $[r]$. The adversary then computes $B = \beta(f)$, sets $w^\star = k^\star \oplus e_{B}$, and appends the set
\begin{align*}
    U := \biguplus_{t \in [r]} f_t \cdot \{ (x_t(w^\star), (r, w^\star)) \}
\end{align*}
to the clean sample $C$. Thus, the adversary appends exactly $n$ corrupted points. Since any $t$ for which $f_t > 0$ satisfies that $t \notin B$, by part (2) of \Cref{claim:same-labels-different-hypotheses}, the enlarged dataset is consistent with the target hypothesis $h^\star=h_{r, k^\star}$. Moreover, the resulting dataset (in an unordered, multiset sense) is exactly $T_{k^\star, F}$. We therefore have, from \Cref{prop:k-choice}, that
\begin{align*}
    \E \, \err_{D, h^\star}(\Lcal(\Shuf(C \cup U)))&= \E_{f}\left[E_{\Pi, \sigma} \left[\err_{D_{k^\star}, h_{r,k^\star}}(\Lcal(\Pi(T_{k^\star,f})))\right]\right] \ge \frac14.
\end{align*}
This concludes the lower bound.

\subsection{DS one lower bounds}\label{app:portals}

The construction above gives complete learnability failure at a linear corruption budget. We now prove the
quantitative logarithmic sample complexity blowup result for DS-one classes from \Cref{sec:portal-main}.

As a preliminary, we will require the following standard coupon-collector bound:


\begin{lemma}[Coupon Collector]\label{lem:unseen-colors}
Let $U$ be the number of unseen numbers from $[q]$ after $N$ independent uniformly random draws
from $[q]$. Then $\E[U]=q(1-1/q)^N$ and
$\operatorname{Var}(U)\le\E[U]$. Consequently, if $q \ge 8$ and $N\le (q\log q)/4$, then
$\Pr[U>0]\ge3/4$.
\end{lemma}
\begin{proof}
    For $c \in [q]$, let $I_c = \Ind[c \text{ is not observed in the $N$ samples}]$; we have $$U = \sum_{c \in [q]} I_c, \quad \E[U] = q\left(1-\frac{1}{q}\right)^N, \quad \mathrm{Var}(I_c)=\left(1-\frac{1}{q}\right)^N\left(1-\left(1-\frac{1}{q}\right)^N\right)\le \left(1-\frac{1}{q}\right)^N.$$
    Now, for distinct $c, c'$, we have that
    \begin{align*}
        &\E[I_cI_c'] = \Pr[\text{both $c,c'$  not observed in the $N$ samples}] = \left(1-\frac{2}{q}\right)^N \le \left(1-\frac{1}{q}\right)^{2N} = \E[I_c]\E[I_{c'}] \\
        &\implies \quad \mathrm{Cov}(I_c, I_c) \le 0.
    \end{align*}
    Thus, we have that
    \begin{align*}
        \mathrm{Var}[U] &= \sum_{c} \mathrm{Var}[I_c] + 2\sum_{c < c'} \mathrm{Cov}(I_c,I_c') \le q\left(1-\frac{1}{q}\right)^N = \E[U].
    \end{align*}
    Now, note that
    \begin{align*}
        \log\left(1-\frac{1}{q}\right) &= -\log\left(1+\frac{1}{q-1}\right) \ge -\frac{1}{q-1} \ge -\frac{4}{3q}, \tag{using $\log(1+t)\le t$ for all $t \ge 0$, and $q \ge 4$}
    \end{align*}
    and hence
    \begin{align*}
        \E[U] = q\left(1-\frac{1}{q}\right)^N =q \exp\left(N\log\left(1-\frac{1}{q}\right)\right) \ge q\exp\left(-\frac{4N}{3q}\right) \ge q^{2/3}.\tag{using $N \le (q\log q)/4$}
    \end{align*}
    Finally, Chebyshev's inequality gives that
    \begin{align*}
        \Pr[U=0] \le \Pr[|U-\E[U]| \ge \E[U]] \le \frac{\mathrm{Var}[U]}{\E[U]^2} \le \frac{1}{\E[U]} \le q^{-2/3} \le \frac{1}{4}, \tag{using that $q \ge 8$}
    \end{align*}
    and hence $\Pr[U > 0] \ge 3/4$ as required.
\end{proof}

We are now ready to specify the DS-one class that witnesses the lower bound. The class will constructed from a structured bipartite graph.

\paragraph{The DS-one class.}

For a prime number $q$, let $G_q=(L_q,R_q,E_q)$ be a bipartite graph, where $L_q = \F_q^2$, $R_q = \F_q^2$, and a pair $(x,y) \in L_q$ is connected to a pair $(a,b) \in R_q$ iff $y=ax+b$. Note that this graph is $q$-regular (since, e.g., for any $(a,b) \in R_q$, fixing $x \in \F_q$ uniquely determines $y=ax+b$) and $C_4$-free (since a $C_4$ would imply two distinct lines sharing two distinct points). By Vizing's theorem, its edges admit a proper edge coloring $\chi:E_q \to [q]$ such that all $q$ edges incident to any vertex are colored differently. For any vertex $v$ in the graph, let $E(v)$ denote the $q$ edges incident to it, and let $N(v)$ denote its neighbors.

The domain of the class is the set of edges $E_q$. There is a hypothesis $h_v$ in the class corresponding to every vertex $v$ in the graph, and so, $\H_q = \{h_v: v \in L_q \cup R_q\}$. The hypothesis $h_v$ labels any edge $e=uw$, where $u \in L_q$ and $w \in R_q$ as follows:
\[
    h_v(e)
    =
    \begin{cases}
        0,&u=v \text{ or } v \in N(u) \setminus \{w\},\\
        1,&v=w \text{ or } v \in N(w) \setminus \{u\},\\
        \lambda_{e,v},& \text{otherwise},
    \end{cases}
\]
where $\lambda_{e,v}$ is a label unique to the pair $(e,v)$, and used by no other hypothesis at edge $e$.

The crucial property of this class is the following: for any edge $e=uv$, $h_u(e) \neq h_v(e)$, but for any other edge $e' \in E(u) \setminus \{e\}$, it holds that $h_u(e')=h_v(e')=0$, and also for any edge $e' \in E(v) \setminus \{e\}$, it holds that $h_u(e')=h_v(e')=1$.

We now argue that the DS dimension of this class is 1. By the result of \cite{pabbaraju2026optimal}, this implies that $R^{\imp}_{\clean}(n;\H)=O(1/n)$.

\begin{lemma}\label{lem:portal-ds1}
For every $q\ge2$, $d_{\DS}(\H_q)=1$.
\end{lemma}

\begin{proof}
The hypotheses corresponding to the endpoints of any edge realize labels $0$ and $1$ at this edge, so the DS dimension is at least one. Suppose now that there exists a set of two points that is DS-shattered by the class. This implies the existence of a finite, non-empty set of label pairs in which each pair can change either coordinate while keeping the other coordinate fixed.

Now observe that no pair in such a set can have a unique label $\lambda_{v,e}$, since this pins down the hypothesis to be $h_v$, and hence, we cannot change the other coordinate while keeping this coordinate fixed. So, we necessarily have that the set of label pairs is 
\begin{align*}
    (0,0), (0,1), (1,0), (1,1).
\end{align*}
Now suppose that the two points correspond to disjoint edges $uv$ and $wx$. In this case, observe that in order to realize the label pair $(0,1)$ on these two edges, the edge $ux$ must exist in the graph. Similarly, in order to realize the label pair $(1,0)$ on these two edges, the edge $vw$ must exist in the graph. But together with the edges $uv$ and $wx$, this forms a $C_4$ in the graph, which is forbidden. Thus, in this case, the label pair $(0,1)$ and $(1,0)$ cannot be simultaneously realized.

On the other hand, suppose that the two points correspond to edges $uv$ and $uw$ that share the left endpoint $u$. In this case, observe that in order to realize the label pair $(1,1)$, both $v$ and $w$ must be adjacent to a common left vertex $x \neq u$, implying a $C_4$ again. Symmetrically, if the two edges share a right endpoint, we get that the label pair $(0,0)$ implies a $C_4$.

In all cases, we have argued that the two points cannot be shattered. Hence, the DS dimension is 1.
\end{proof}

\begin{proofof}{\Cref{thm:portal-degradation}}
    For any vertex $v$, let $D_v$ be the uniform distribution over $E(v)$. Observe then that for any edge $e=uv$, $h_u(e)=0$ and $h_v(e)=1$, and so, for any deterministic predictor $g$,
    \begin{align*}
        \err_{D_u, h_u}(g) + \err_{D_v, h_v} \ge \frac{1}{q}.
    \end{align*}
    The lower bound will be witnessed by some hypothesis $h_v$ and data distribution $D_v$. We first specify the action of the monotone adversary. Let $n = \Theta(q\log q)$, and let $C$ be a clean dataset of $n$ edges labeled by $h_v$. Let $f$ be the ``color profile'' of the $n$ edges in $C$ (this is the frequency vector of each color $c$). If the dataset comprises of all $q$ edges incident to $v$, the adversary appends an arbitrary set of $n$ points labeled by $h_v$. Otherwise, $f_c=0$ for some $c$; choose the smallest such $c=c(f)$ and let the other endpoint of the edge incident to $v$ having color $c$ be $w$. By the properties of the class, $C$ is also realizable by $h_w$. 

    In this case, for each color $t$ where $f_t > 0$, the adversary appends the edge incident to $w$ also having color $t$, and labels this edge according to $h_w$. Namely, let $e_t(v)$ be the edge having color $t$ adjacent to $v$. Define
    \begin{align}
        T_{v,f} := \biguplus_{t \in [q]} 
        f_t \cdot \left\{\left(e_t(v), h_v(e_t(v))\right), \left(e_t(w), h_w(e_t(w))\right)\right\}.
    \end{align}
    We have that $T_{v,f}=T_{w,f}$ and both $T_{v,f}$ and $T_{w,f}$ are realizable by both $h_v$ and $h_w$. So, for any learner $\Lcal$ possibly using randomness $\sigma$, by arguments similar to those in \Cref{sec:monotone-adversary-action}, we get that
    \begin{align*}
        &\E_{\Pi, \sigma}[\err_{D_v, h_v}(\Lcal(\Pi(T_{v,f})] + \E_{\Pi, \sigma}[\err_{D_w, h_w}(\Lcal(T_{w,f}))] \ge \frac{1}{q},
    \end{align*}
    and also that, for $f$ distributed as the color frequency vector of $n$ uniformly random draws from $[q]$ conditioned on some color being unseen, there exists $v^\star$ which only depends on $n$ and $\Lcal$ satisfying
    \begin{align*}
        \E_f\left[ \E_{\Pi, \sigma}[\err_{D_{v^\star}, h_{v^\star}}(\Lcal(\Pi(T_{v^\star,f})] \right] \ge \frac{1}{2q}.
    \end{align*}
    To finish up, let $C$ be a clean sample of $n$ i.i.d.\ draws from $D_{v^\star}$ and labeled by $h_{v^\star}$. By \Cref{lem:unseen-colors}, with probability $3/4$, we have that some color is unseen in $C$. The adversary obtains the color frequency vector $f$, computes the smallest unseen color $c=c(f)$, obtains the vertex $w$ adjacent to $v^\star$ along the edge colored $c$, and appends the set 
    \begin{align*}
        U := \biguplus_{t \in [q]} 
        f_t \cdot \left\{\left(e_t(w), h_w(e_t(w))\right)\right\}
    \end{align*}
    to the clean sample $C$. By the analysis above, the resulting dataset, which is exactly $T_{v^\star, f}$, is realizable by $h_{v^\star}$, and we get
    \begin{align*}
        \E \, \err_{D_{v^\star}, h_{v^\star}}(\Lcal(\Shuf(C \cup U)))&\ge \frac{3}{4} \E_{f}\left[E_{\Pi, \sigma} \left[\err_{D_{v^\star}, h_{v^\star}}(\Lcal(\Pi(T_{v^\star,f})))\right]\right] \ge \frac{3}{8q} = \Omega\left(\frac{\log n}{n}\right).
    \end{align*}    
\end{proofof}

\section{Partial Classes Lower Bound}
\label{sec:partial-class-lower-bound-rewrite}

In this section, we restate and prove \Cref{thm:partial-destruction}. Namely, we will construct a binary partial hypothesis class that is learnable with clean i.i.d.\ data, but is not learnable under adaptive monotone adversarial corruptions.

\PartialLowerBound*


\subsection{The Hard Class $\Hcal$}
\label{sec:hard-partial-class}

\paragraph{Preliminaries.} 
Similar to \Cref{sec:hard-class}, let $r \ge 2$ be any even integer, and let $\Bcal_r=\binom{[r]}{r/2}$ be the collection of all half-sets of $[r]$. For every $r$, choose a prime $q$ large enough so that we can choose pairwise distinct elements
\begin{align*}
    z_1,\dots,z_r,\quad (\alpha_B)_{B \in \Bcal_r},
\end{align*}
within $\F_q$ (the additive group modulo $q$). In particular, the smallest prime $q$ larger than $r + \binom{r}{r/2}$, which, by Bertrand's postulate, is at most $2\left(r + \binom{r}{r/2}\right)$, suffices.

We define the affine function $\phi_B:\F_q \to \F_q$, for every $B \in \Bcal_r$, as
\begin{align}
    \phi_B(z) = z - \alpha_B.
\end{align}

\paragraph{The Marker Flower.} In order to specify the domain $\Xcal$ of the class, we will construct a special set consisting of ``markers'', and these markers will be  
arranged in the form of a ``flower''. Namely, we will first have a common marker $\perp$. Then, for every halfset $B \in \Bcal_r$, which will signify a ``direction'', we will have a \textit{petal} $e_B$ in this direction, which comprises the markers
\begin{align}
    e_B := \{e_{B, a}: a \in \F_q \setminus \{0\}\}.
\end{align}
Note that all the $e_B$'s are pairwise disjoint. For notational convenience, we will adopt the convention that $e_{B,0}=\perp$ for every $B$. In total, the marker flower $\Omega_r$ comprises
\begin{align*}
    \Omega_r := \{\perp\} \sqcup \bigsqcup_{B \in \Bcal_r} e_B.
\end{align*}
We next define, for any $B \in \Bcal_r$ and $a \in \F_q$, the ``petal-rotator'' map $\tau_{B, a}:\Omega_r \to \Omega_r$, which acts on any $e_{B', c}$ as
\begin{align}
    \label{eqn:def-petal-rotator}
    \tau_{B,a}(e_{B',c}) = \begin{cases}
        e_{B',c+a} & \text{if $B'=B$}, \\
        e_{B',c} & \text{otherwise}.
    \end{cases}
\end{align}
So, $\tau_{B,a}$ rotates precisely the elements of the petal $e_B$, along with the common marker $e_{B,0}=\perp$, by an amount $a$. Observe that $\tau_{B,a}$ is always a bijection, and $\tau_{B,a} \circ \tau_{B,b} = \tau_{B, a+b}$.

\paragraph{$t$-cell.} We will now define the notion of a \textit{$t$-cell} --- this is similar to the notion of $t$-masked vectors that we considered in the multiclass lower bound. For any ``tag'' $t \in [r]$, define 
\begin{align}
    \label{eqn:visible-petals}
    O_t := \bigsqcup_{B \not\ni t} e_B.
\end{align}
That is, $O_t$ contains precisely all the elements of petals $e_B$ for which the halfset $B$ \textit{does not} contain $t$; additionally it does not contain $\perp$.

\begin{definition}[$t$-cell]
A $t$-cell $F$ is any injective function mapping $O_t$ to $\Omega_r$.
\end{definition}

\paragraph{States and Corrections at Cells.}
Next, we define the notion of a \textit{state}. A state is any pair $w=(\sigma, p)$, where $\sigma:\Omega_r \to \Omega_r$ is a bijection, and $p$ is an affine function mapping $\F_q \to \F_q$, i.e., $p(z)=uz+v$, $u,v \in \F_q$.

We say that a state $w=(\sigma, p)$ is correctable\footnote{This notion is faintly similar to how we compared binary vectors $k$ with $t$-masked vectors $a$ in the multiclass lower bound.} at a $t$-cell $F$ if either 
\begin{align*}
    \sigma|_{O_t} = F,
\end{align*}
or there exists $B \not\ni t$, and $a \neq 0$ such that
\begin{align*}
    (\sigma \circ \tau_{B, -a})|_{O_t} = F.
\end{align*}
In the former case, ``no correction is needed'' --- $\sigma$ restricted to $O_t$ is already equal to $F$ --- and the \textit{corrected state} is defined to be $(\sigma, p)$ itself. In the latter case, we need to first undo a rotation on a petal $B$ by an amount $a$; in this case, we say that ``$w$ is correctable with $(B,a)$'', and the corrected state is defined to be 
\begin{align*}
    (\sigma \circ \tau_{B, -a}, p-a\phi_B).
\end{align*}
Now suppose that a state $(\sigma, p)$ is correctable at a $t$-cell $F$ 
with $(B,a)$ where $B \not\ni t$ and $a \neq 0$. Let $y = \sigma(\perp)$. Since the correction uses $(B,a)$, we have that
\begin{align*}
    (\sigma \circ \tau_{B, -a})(e_{B,a}) = \sigma(e_{B,0}) = \sigma(\perp) = y = F(e_{B,a}).
\end{align*}
So, if $(\sigma,p)$ were simultaneously also correctable with a different $(B', a')$, we would get that $y=F(e_{B,a})=F(e_{B',a'})$, which would contradict injectivity of $F$. Thus, whenever a correction is needed for $(\sigma, p)$, $y=\sigma(\perp)$ belongs to the range of $F$, and uniquely identifies the corrective term $(B,a)$.

On the other hand, if $(\sigma, p)$ is correctable at a $t$-cell $F$ with no correction needed, then $\sigma(O_t)=F(O_t)$. This means that $y \notin F(O_t)$; otherwise, we would have that $\sigma(x) = y$ for some $x \in O_t$, and since $\perp \notin O_t$, this breaks injectivity of $\sigma$. Thus, a state that is correctable with no correction cannot also simultaneously be correctable with a non-trivial correction. Furthermore, when no correction is needed, it holds that $y \notin F(O_t)$.


\paragraph{The Hard Partial Concept Class.} We are now ready to define the partial concept class for the lower bound.

For every two-element subset $P \subseteq \F_q$, where $|P| = 2$, fix once and for all a bijection
\begin{align}
    \label{eqn:two-element-subset-bijection}
    \lambda_P: P \to \{0,1\}.
\end{align}
Define $\Xcal_r$ --- the domain of the hypothesis class at level $r$ --- to be the set
\begin{align}
    \Xcal_r := \{(r, t, F, P): t \in [r], F:O_t \hookrightarrow \Omega, P \subseteq \F_q, |P|=2\}.
\end{align}
The complete domain $\Xcal$ of the class is then defined to be
\begin{align}
    \Xcal := \bigsqcup_{\substack{r \ge 2 \\ r \text{  even }}} \Xcal_r.
\end{align}
The hypothesis class $\Hcal$ will have a partial concept $h_{r,w}$ for every $r$ and every state $w=(\sigma,p)$. First, $h_{r,w}$ maps any $\Xcal_s$ with $s \neq r$ entirely to $*$. For any $x=(r,t,F,P) \in \Xcal_r$, its behavior is as follows:
\begin{align*}
    h_{r,w}((r,t,F,P)) := \begin{cases}
        \lambda_P(p'(z_t)) & \text{if $w$ is correctable at $F$ with corrected state $(\sigma', p')$ and $p'(z_t) \in P$}, \\
        * & \text{otherwise}.
    \end{cases}
\end{align*}
That is, we first check if the state $w$ is correctable at the $t$-cell $F$; if so, we consider the corrected state $(\sigma', p')$, and check if $p'(z_t) \in P$, i.e., whether $p'(z_t)$ belongs to the domain of $\lambda_P$. If this is indeed the case, the label is set to be $\lambda_P(p'(z_t))$.

\subsection{Bounding the Partial VC Dimension} 

We will now argue that the partial concept class constructed above has constant VC dimension. In particular, we will show that no set of 19 points from the domain can be shattered, and hence the VC dimension is at most 18. This will establish part (i) of the theorem, since from the results of \citet{alon2022theory, aden2023optimal}, it is known that for a partial binary class of partial VC dimension $d$, $N^\imp_\clean(\epsilon, \delta; \H)=O((d+\log(1/\delta))/\epsilon)$. 

In what follows, fix any set $S=\{x_1,\dots,x_{19}\} \subseteq \Xcal$ of distinct points. Recall that each $x_i=(r_i, t_i, F_i, P_i)$. We will count the number of distinct patterns from $\{0,1\}^{19}$ that may be realized by $\Hcal$ on $S$ --- our goal will be to show that this number is smaller than $2^{19}$, implying that $S$ is not shattered. 

Observe first that in order for $S$ to be shattered, it must be the case that $S \subseteq \Xcal_r$ for some $r$; otherwise, every hypothesis in $\Hcal$ necessarily labels some point in $S$ with $*$.

Now, fix some reference hypothesis $h^\star = h_{r, w^\star}$ corresponding to the state $w^\star = (\sigma^\star, p^\star)$, which realizes non-$*$ labels on all of $S$. We will count the number of other distinct (non-$*$) patterns that may be realized by other hypotheses $h_{r, (\sigma, p)} \in \Hcal$ on $S$, in reference to $h^\star$.

Towards this, consider any $h_{r, (\sigma, p)} \in \Hcal$ that is non-$*$ on all of $S$, and let $p(z)=uz+v$. Recall that for $h_{r, (\sigma, p)}$ to be non-$*$ on $x_i$, it must be the case that $(\sigma,p)$ is correctable at $F_i$. Furthermore, in this case, the value $c_i = p'_i(z_{t_i})$ corresponding to the corrected state $(\sigma'_i, p'_i)$ determines the 0/1 label at $x_i$.  It therefore suffices to reason about the number of distinct ``color'' patterns $\{c_1,\dots,c_{19}\}$ that may be realized on $S$.

Recall further from our reasoning above that when $(\sigma, p)$ is correctable at $F_i$, it suffices to reason about $\sigma(\perp)$ to determine whether $(\sigma, p)$ needs a non-trivial correction: if $\sigma(\perp)$ does not belong to the range of $F_i$, then no correction is required, and if $\sigma(\perp)=F_i(e_{B,a})$, then the correction is $(B,a)$.

Keeping this in mind, we will case on how $\sigma(\perp)$ compares to $\sigma^\star(\perp)$. 

\paragraph{Case 1: $\sigma(\perp)=\sigma^\star(\perp)$.} In this case, $(\sigma, p)$ needs a non-trivial correction at precisely those $F_i$ where $(\sigma^\star, p^\star)$ needs a correction. So, suppose that $(\sigma^\star, p^\star)$ is correctable with $(B^\star_i, a^\star_i)$ at $F_i$; here, $a^\star_i=0$ captures the case where no correction is required at $F_i$. Then, we have that the value $c_i$ is equal to
\begin{align*}
    c_i = p'_i(z_{t_i})= uz_{t_i} + v - a^\star_i(z_{t_i}-\alpha_{B^\star_i}).
\end{align*}
In the above, $u$ and $v$ are free, whereas the rest of the quantities are fixed. However, since $h_{r, (\sigma, p)}$ is non-$*$ at every $x_i$, it must also be the case that $c_i \in P_i$. So, we are in a situation where each of $c_1,\dots,c_{19}$ is an affine function of two free variables, and also that every $c_i$ is restricted to be one of two possible values from $P_i$. Consequently, the number of distinct color patterns that can be realized are at most $2^2 = 4$ (this is because, fixing, say $c_1$ and $c_2$, determines the rest of the $c_i$s).

Summarily, we have that at most $4$ distinct color patterns are realized on $S$ by hypotheses $h_{r, (\sigma, p)}$ that satisfy $\sigma(\perp)=\sigma^\star(\perp)$. Note that this includes the hypothesis $h^\star=h_{r, (\sigma^\star, p^\star)}$.

\paragraph{Case 2: $\sigma(\perp)\neq \sigma^\star(\perp)$.} Since $\sigma^\star$ is a bijection, it must be so that $\sigma(\perp)=\sigma^\star(e_{B,a})$ for some $B \in \Bcal_r$, $a \neq 0$. Consider any $x_i=(r, t_i, F_i, P_i)$:

\paragraph{Subcase 2A: $t_i \in B$.} We claim in this case that $(\sigma, p)$ needs no correction at $F_i$, and hence, $c_i=p(z_{t_i})$.

To see why, observe first that since $t_i \in B$, $e_{B,a} \notin O_{t_i}$, meaning that $e_{B,a}$ is not in the domain of $F_i$. Now, since $\sigma^\star$ is non-$*$ at $x_i$, 
\begin{align*}
\text{either $F_i=\sigma^\star|_{O_{t_i}}$,} \qquad \text{or} \qquad \text{$F_i = (\sigma^\star \circ \tau_{B^\star_i, -a^\star_i})|_{O_{t_i}}$ for $t_i \notin B^\star_i, a^\star_i \neq 0$}.
\end{align*}
In the former case, note that since $\sigma^\star$ is a bijection, the only input that $\sigma^\star$ sends to $\sigma^\star(e_{B,a})$ is $e_{B,a}$; since $e_{B,a} \notin O_{t_i}$, and $F_i=\sigma^\star|_{O_{t_i}}$, we conclude that $\sigma^\star(e_{B,a})=\sigma(\perp) \notin F_i(O_{t_i})$. This means that $(\sigma, p)$ needs no correction at $F_i$.

In the latter case, note that since $t_i \in B$, we have that $B \neq B^\star_i$. Furthermore, since $a \neq 0$, we can conclude that
\begin{align*}
    \tau_{B^\star_i, -a^\star_i}(e_{B,a})=e_{B,a} \implies (\sigma^\star \circ \tau_{B^\star_i, -a^\star_i})(e_{B,a}) = \sigma^\star(e_{B,a}).
\end{align*}
Then, since  $\sigma^\star \circ \tau_{B^\star_i, -a^\star_i}$ is also a bijection, the only input that it sends to $\sigma^\star(e_{B,a})$ is $e_{B,a}$, meaning again that $\sigma^\star(e_{B,a})=\sigma(\perp) \notin F_i(O_{t_i})$, and $(\sigma, p)$ needs no correction at $F_i$.

\paragraph{Subcase 2B: $t_i \notin B$, but $(\sigma^\star, p^\star)$ does not use $B$ as the correction at $F_i$.} That is, $(\sigma^\star, p^\star)$ uses a correction $(C,d)$ at $x_i$, where $C \neq B$, and $d=0$ also captures the case where no correction is required.

In this case, observe that
\begin{align*}
    F_i(e_{B,a}) = (\sigma^\star \circ \tau_{C, -d})(e_{B,a})=\sigma^\star(e_{B,a})=\sigma(\perp),
\end{align*}
and hence, the correction that $(\sigma, p)$ uses at $F_i$ is precisely $(B,a)$.

\paragraph{Subcase 2C: $t_i \notin B$, and $(\sigma^\star, p^\star)$ uses $(B,a^\star_i)$ as the correction at $F_i$ for some $a^\star_i \neq 0$.}

First, suppose that $a+a^\star_i=0$; in this case, we claim that $\sigma(\perp)=\sigma^\star(e_{B,a})$ does not belong to the range of $F_i$, and hence $(\sigma, p)$ does not require a correction at $F_i$. To see this, suppose otherwise that $F_i(e_{C,d})=\sigma(\perp)$ for some $e_{C,d} \in O_{t_i}, d \neq 0$. Then, since $(\sigma^\star, p^\star)$ uses correction $(B, a^\star_i)$ at $F_i$, we have that
\begin{align*}
    \sigma(\perp)=F_i(e_{C,d}) = (\sigma^\star \circ \tau_{B, -a^\star_i})(e_{C, d}) = \begin{cases}
        \sigma^\star(e_{C,d}) & \text{if $B \neq C$}, \\
        \sigma^\star(e_{B,d-a^\star_i})=\sigma^\star(e_{B,a+d}) & \text{if $B=C$}.
    \end{cases}
\end{align*}
In either case, the injectivity of $\sigma^\star$ is contradicted, since $\sigma^\star(e_{B,a})=\sigma(\perp)$.

Otherwise, we have that $a+a^\star_i\neq 0$; in this case, observe that
\begin{align*}
    F_i(e_{B, a+a^\star_i}) = (\sigma^\star \circ \tau_{B, -a^\star_i})(e_{B, a+a^\star_i}) = \sigma^\star(e_{B,a}) = \sigma(\perp),
\end{align*}
 and hence, the correction that $(\sigma, p)$ uses at $F_i$ is precisely $(B,a+a^\star_i)$.

\Cref{table:case-2-summary} below summarizes the analysis above for Case 2.

\begin{table}[H]
\centering
\begin{tabular}{|c|c|}
\hline
\textbf{Situation at $F_i$}
&
\textbf{Correction for $(\sigma, p)$}
\\
\hline
$t_i \in B$
&
none
\\
$t_i \notin B$, reference correction does not use $B$
&
$(B,a)$
\\
$t_i \notin B$, reference correction is $(B,a^\star_i)$
&
$(B,a+a^\star_i)$
\\
\hline
\end{tabular}
\caption{Summary of Case 2.}
\label{table:case-2-summary}
\end{table}
In the last row, $a+a^\star_i=0$ stands for no correction.

Now, define 
\begin{align*}
    \Ecal=\{B \in \Bcal_r: (\sigma^\star, p^\star) \text{ uses $(B,a)$ as the correction at some $F_i$ in $S$, where $a \neq 0$,}\}.
\end{align*}
That is, $\Ecal$ comprises (only the $\Bcal_r$ component of) all the non-trivial corrections that the reference uses on $S$. Note that $|\Ecal| \le 19$.

Suppose now that $\sigma(\perp)=\sigma^\star(e_{B,a})$ for some $B \in \Ecal$, $a \neq 0$. We will count the number of different color patterns $c_1,\dots,c_{19}$ that can be realized.

Fix a $B \in \Ecal$; recall that $(\sigma'_i, p'_i)$ is the corrected state at $F_i$. We have, using \Cref{table:case-2-summary}, that
\begin{align*}
    c_i = \begin{cases}
    uz_{t_i}+v & \text{Subcase 2A}, \\
    uz_{t_i}+v - a(z_{t_i}-\alpha_{B}) & \text{Subcase 2B}, \\
    uz_{t_i}+v - (a+a^\star_i)(z_{t_i}-\alpha_{B}) & \text{Subcase 2C.}
    \end{cases}
\end{align*}
Note in the above that the free variables are $a$, $u$ and $v$; the rest of the variables are fixed. Furthermore, $c_i$ is constrained to belong to $P_i$. We thus have that each of $c_1,\dots,c_{19}$ is an affine function of three free variables, and also that every $c_i$ is restricted to be one of two possible values from $P_i$. Consequently, the number of distinct color patterns that can be realized are at most $2^3 = 8$ (this is because, fixing, say $c_1, c_2$ and $c_3$, determines the rest of the $c_i$s).

Varying $B$ over $\Ecal$, which has at most 19 elements, we obtain that the total number of distinct color patterns realizable on $S$ by hypotheses $h_{r, (\sigma, p)}$ that satisfy $\sigma(\perp)=\sigma^\star(e_{B,a})$ for $a \neq 0, B \in \Ecal$ is at most $19 \times 8 = 152$.

In total, the number of color patterns from hypotheses $h_{r, (\sigma, p)}$ satisfying either that $\sigma(\perp)=\sigma^\star(\perp)$ or $\sigma(\perp)=\sigma^\star(e_{B,a})$ for $a \neq 0$, $B \in \Ecal$ is at most $156$.

It now remains to reason about the case where $\sigma(\perp)=\sigma^\star(e_{B,a})$ for $a \neq 0, B \notin \Ecal$. In this case, for every $F_i$, either Subcase 2A or 2B applies, and hence, we have that
\begin{align*}
    c_i = \begin{cases}
    uz_{t_i}+v & t_i \in B, \\
    uz_{t_i}+v - a(z_{t_i}-\alpha_{B}) & t_i \notin B.
    \end{cases}
\end{align*}
Let $p_1=p$ and $p_0=p-a\phi_B$. Both $p_1$ and $p_0$ are affine functions, and since $a \neq 0$, $p_1 \neq p_0$. Rewriting the above in terms of $p_1$ and $p_0$, we have
\begin{align*}
    c_i = \begin{cases}
    p_1(z_{t_i}) & t_i \in B, \\
    p_0(z_{t_i}) & t_i \notin B.
    \end{cases}
\end{align*}
Let
\begin{align*}
    I_B := \{i \in [19]: t_i \in B\}.
\end{align*}
The coordinates in $I_B$ are precisely those that use $p_1$; the other coordinates use $p_0$.

We now observe that the pair $(p_1,p_0)$ determines $(B,a)$. To see this, note that
\begin{align*}
    p_1(z)-p_0(z)=a\phi_B(z)= az - a\alpha_B.
\end{align*}
Thus, the slope of $p_1-p_0$ is equal to $a$, and its zero is equal to $\alpha_B$ (which then determines $B$). So, given $(p_1,p_0)$, we can determine $(B,a)$, and thereafter determine $c_i$ for every $i$, depending on whether $t_i \in B$. Thus, it suffices to count the number of admissible pairs $(p_1,p_0)$.

Let us first handle the case where $t_i = t_j$ for some $i \neq j$. In this case, observe that $c_i = c_j$. Thus, the color pattern is determined by one color for every distinct $t_i$. Further, since each $t_i$ is constrained to belong to $P_i$, and since there are at most $18$ distinct numbers in $t_1,\dots,t_{19}$, we can have at most $2^{18}$ color patterns. Thus, combining with the 156 color patterns from the other cases above, we obtain that in total, the number of color patterns realizable on $S$ is at most $2^{18} + 156 < 2^{19}$.

Now, suppose that every $t_i$ is distinct, meaning that $z_{t_1},\dots,z_{t_{19}}$ are all distinct. Define
\begin{align*}
    \Acal := \{f \text{ affine in } z: f(z_{t_i}) \in P_i \text{ for at least two indices $i$}\}.
\end{align*}
We argue that $\Acal$ is small; to see this, take any $f \in \Acal$. There are at least two indices $i,j$ for which $f(z_{t_i}) \in P_i$ and $f(z_{t_j}) \in P_j$. There are $\binom{19}{2}$ choices for $i,j$, and for a given choice, there are $2 \times 2=4$ choices for the values from $P_i$ and $P_j$. Fixing $f$'s evaluation at the two distinct points $z_{t_i}$ and $z_{t_j}$ uniquely determines $f$. Thus, each choice of a pair of indices and values at those indices determines at most one function in $\Acal$. We thus have that
\begin{align*}
    |\Acal| \le 4\binom{19}{2} = 684.
\end{align*}
Let us now consider values $B$ for which $|I_B| \ge 2$ and $19-|I_B| \ge 2$, meaning that both $p_1$ and $p_0$ are used to determine a color at least twice. Since our hypothesis is non-$*$ everywhere, for both of the (at least) two coordinates that $p_1$ determines a color for, it evaluates to a value within the corresponding $P_i$s, and likewise for $p_0$. We thus have that $p_1 \in \Acal, p_0 \in \Acal$. Thus, the number of choices for $(p_1,p_0)$ is at most $|\Acal|^2 \le 684^2$. In summary, the total number of possible color patterns for these $B$ values is at most $684^2$.

The remaining values of $B$ satisfy either $|I_B| \le 1$ or $19-|I_B| \le 1$. Thus, one of $p_1$ or $p_0$ determines at least 18 colors --- call this function $f$. Certainly, $f \in \Acal$. Thus, there are at most $684$ choices for $f$.

Now, after $f$ is chosen, the color pattern $c_1,\dots,c_{19}$ agrees with $f(z_{t_1}),\dots,f(z_{t_{19}})$ at every coordinate except possibly one. If the color pattern agrees everywhere, we obtain one color pattern. Otherwise, we have 19 choices for the coordinate $j$ where the color pattern disagrees, and at this $j$, we can choose one of two colors from $P_j$. This contributes $19 \times 2 = 38$ patterns. In total, we obtain $1+38=39$ possible color patterns for this $f$. Varying over the 684 choices for $f$, we obtain $684\times 39=26676$ total color patterns.

Thus, when every $t_i$ is distinct, we obtain that the number of total color patterns contributed by hypotheses $h_{r, (\sigma,p)}$ that satisfy $\sigma(\perp)=\sigma^\star(e_{B,a})$ for $a\neq 0, B \notin \Ecal$ is at most $684^2 + 26676=494532$. Combining with the 156 color patterns from the other two cases above we obtain a total of $494532 + 156 < 2^{19}$ color patterns.

We have thus argued that in all cases, the number of total color patterns (and hence the number of binary patterns) realizable on any 19-element set $S$ is strictly smaller than $2^{19}$. The partial VC dimension of the hypothesis class is hence at most 18.

\subsection{The Indistinguishable Worlds.}

We begin by setting up some notation for convenience: for a state $w=(\sigma, p)$ and $t \in [r]$, let
\begin{align*}
    F_t(w)=\sigma|_{O_t}, \qquad c_{w,t}=p(z_t).
\end{align*}
Now, for any $d \in \F_q \setminus \{0\}$, let
\begin{align*}
    x_{t,d}(w) = (r,t,F_t(w), \{c_{w,t}, c_{w,t}+d\}).
\end{align*}
Then, observe that
\begin{align*}
    h_{r,w}(x_{t,d}(w)) = \lambda_{\{c_{w,t}, c_{w,t}+d\}}(c_{w,t}).
\end{align*}
Define the distribution $D_w$ by choosing $(t,d)$ uniformly over $[r] \times (\F_q \setminus \{0\})$ and returning $x_{t,d}(w)$.

Next, for any $w=(\sigma,p)$, and $B \in \Bcal_r$, define the $B$-line through $w$ to be the set
\begin{align*}
    L_B(w) := \{(\sigma \circ \tau_{B,a}, p+a\phi_B): a \in \F_q\}.
\end{align*}
We have the $B$-line through $w$ has exactly $q$ members.

\begin{claim}
    \label{claim:same-labels-different-hypotheses-partial}
    Fix a state $w=(\sigma,p)$ and direction $B \in \Bcal_r$. Then, for any $t \in [r]$:
    \begin{enumerate}
        \item[(1)] If $t \in B$, then $F_t(u)=F_t(v)$ for every $u, v \in L_B(w)$. 
        \item[(2)] If $t \notin B$, then for every $u, v \in L_B(w)$, and $d \in \F_q \setminus \{0\}$,
        \begin{align*}
            h_{r,v}(x_{t,d}(u)) = h_{r,u}(x_{t,d}(u)).
        \end{align*}
    \end{enumerate}
\end{claim}
\begin{proof}
    Let $u = (\sigma \circ \tau_{B,a}, p + a\phi_B)$ and $v = (\sigma \circ \tau_{B,b}, p + b\phi_B)$, where $a \neq b$.

    For (1): if $t \in B$, recall that $B \notin O_t$, and since $\tau_{B, \cdot}$ only rotates the petal along direction $B$, $$F_t(u)=(\sigma \circ \tau_{B,a})|_{O_t}=(\sigma \circ \tau_{B,b})|_{O_t}=F_t(v).$$

    For (2): recall that 
    \begin{align*}
        x_{t,d}(u)=(r,t,F_t(u), \{c_{u,t}, c_{u,t}+d\}), \qquad  h_{r,u}(x_{t,d}(u))=\lambda_{\{c_{u,t}, c_{u,t}+d\}}(c_{u,t}).
    \end{align*}
    Now, to determine the label that $h_{r,v}$ assigns $x_{t,d}(u)$, first observe that $v$ is correctable at $F_t(u)$ with correction $(B, b-a)$; in particular, we have that
    \begin{align*}
        \sigma \circ \tau_{B,b} \circ \tau_{B,a-b} = \sigma \circ \tau_{B,a},
    \end{align*}
    which gives us that
    \begin{align*}
        (\sigma \circ \tau_{B,b} \circ \tau_{B,a-b})|_{O_t} = (\sigma \circ \tau_{B,a})|_{O_t} = F_t(u).
    \end{align*}
    The corrected state is precisely $u$, since
    \begin{align*}
        (\sigma \circ \tau_{B,a}, p+b\phi_B-(b-a)\phi_B)=(\sigma \circ \tau_{B,a}, p+a\phi_B)=u.
    \end{align*}
    Thus, 
    \begin{align*}
        h_{r,v}(x_{t,d}(u)) = \lambda_{\{c_{u,t}, c_{u,t}+d\}}(c_{u,t}) = h_{r,u}(x_{t,d}(u)).
    \end{align*}
\end{proof}

Using \Cref{claim:same-labels-different-hypotheses-partial}, we can now argue that every predictor has constant error when averaged over a $B$-line. That is, no predictor can simultaneously have small error in all the worlds $(D_v, h_{r,v})$ within a $B$-line.

\begin{claim}
    \label{claim:many-world-confusion-partial}
    Fix a state $w=(\sigma, p)$ and direction $B \in \Bcal_r$. For any deterministic predictor $g$, it holds that
    \begin{align*}
        \frac{1}{q}\sum_{v \in L_B(w)} \err_{D_v, h_{r,v}}(g) \ge \frac14.
    \end{align*}
\end{claim}
\begin{proof}
    Fix $v \in L_B(w)$, and $d \in \F_q \setminus \{0\}$. Choose $u \in L_B(w)$ satisfying $c_{u,t}=c_{v,t}+d$. Note that for fixed $d$, as $v$ ranges in $L_B(w)$, the corresponding $u$ also ranges in $L_B(w)$. This can be seen as follows: for $v = (\sigma \circ \tau_{B,a}, p+a\phi_B)$, to obtain $u=(\sigma \circ \tau_{B,b}, p+b\phi_B)$ satisfying the condition, we need to solve
    \begin{align*}
        (p+b_{\phi_B})(z_t) = (p+a_{\phi_B})(z_t) + d \implies b = a + \frac{d}{z_t - \alpha_B}.
    \end{align*}
    Since $d \neq 0$, and $z_t-\alpha_B \neq 0$, we have that as $a$ ranges over $\F_q$, $b$ ranges over $\F_q$ as well. Consequently, the corresponding state $u$ ranges over the $B$-line $L_B(w)$.

    But, observe that for a given $v$ and $u$ chosen as above, for $t \in B$, we have 
    \begin{align}
        &x_{t,d}(v) = (r,t,F_t(v), \{c_{v,t}, c_{v,t}+d\}) = (r,t,F_t(v), \{c_{u,t}, c_{v,t}\}). \nonumber \\
        &x_{t,-d}(u) = (r,t,F_t(u), \{c_{u,t}, c_{u,t}-d\}) = (r,t,F_t(v), \{c_{u,t}, c_{v,t}\}), \tag{\Cref{claim:same-labels-different-hypotheses-partial}, part(1)} \\
        \implies\quad&x_{t,d}(v) = x_{t,-d}(u). \label{eqn:same-x}
    \end{align}
    Furthermore,
    \begin{align*}
        h_{r,v}(x_{t,d}(v))= \lambda_{\{c_{u,t}, c_{v,t}\}}(c_{v,t}).
    \end{align*}
    Since $t \in B$, we have that $F_t(u)=F_t(v)$ from \Cref{claim:same-labels-different-hypotheses-partial}, part(1). Hence
    \begin{align*}
        h_{r,u}(x_{t,d}(v)) = 
        h_{r,u}(x_{t,-d}(u))= 
        \lambda_{\{c_{u,t}, c_{v,t}\}}(c_{u,t}) \neq \lambda_{\{c_{u,t}, c_{v,t}\}}(c_{v,t})
    \end{align*}
    since $\lambda_{\{c_{u,t}, c_{v,t}\}}$ is a bijection onto $\{0,1\}$, and thus,
    \begin{align*}
        h_{r,v}(x_{t,d}(v)) \neq h_{r,u}(x_{t,-d}(u)).
    \end{align*}
    So, for any deterministic predictor $g$, using \eqref{eqn:same-x}, we have that
    \begin{align*}
        \Ind\left[ g(x_{t,d}(v)) \neq h_{r,v}(x_{t,d}(v)) \right] + \Ind\left[ g(x_{t,-d}(u)) \neq h_{r,u}(x_{t,-d}(u)) \right] \ge 1.
    \end{align*}
    Summing up over all $v \in L_B(w)$ and $d \neq 0$, we get
    \begin{align*}
        q(q-1) &\le \sum_{v \in L_B(w)} \sum_{d \neq 0} \left(\Ind\left[ g(x_{t,d}(v)) \neq h_{r,v}(x_{t,d}(v)) \right] + \Ind\left[ g(x_{t,-d}(u)) \neq h_{r,u}(x_{t,-d}(u)) \right]\right) \\
        &= \sum_{v \in L_B(w)} \sum_{d \neq 0} \Ind\left[ g(x_{t,d}(v)) \neq h_{r,v}(x_{t,d}(v)) \right] + \sum_{\substack{v \in L_B(w) \\ c_{u,t}=c_{v,t}+d}} \sum_{d \neq 0} \Ind\left[ g(x_{t,-d}(u)) \neq h_{r,u}(x_{t,-d}(u)) \right] \\
        &= \sum_{v \in L_B(w)} \sum_{d \neq 0} \Ind\left[ g(x_{t,d}(v)) \neq h_{r,v}(x_{t,d}(v)) \right] + \sum_{\substack{v \in L_B(w) \\ c_{u,t}=c_{v,t}+d}} \sum_{d \neq 0} \Ind\left[ g(x_{t,d}(u)) \neq h_{r,u}(x_{t,d}(u)) \right] \\
        &= \sum_{v \in L_B(w)} \sum_{d \neq 0} \Ind\left[ g(x_{t,d}(v)) \neq h_{r,v}(x_{t,d}(v)) \right] + \sum_{d \neq 0}\sum_{\substack{v \in L_B(w) \\ c_{u,t}=c_{v,t}+d}}  \Ind\left[ g(x_{t,d}(u)) \neq h_{r,u}(x_{t,d}(u)) \right] \\
        &= \sum_{v \in L_B(w)} \sum_{d \neq 0} \Ind\left[ g(x_{t,d}(v)) \neq h_{r,v}(x_{t,d}(v)) \right] + \sum_{d \neq 0}\sum_{v \in L_B(w) }  \Ind\left[ g(x_{t,d}(v)) \neq h_{r,v}(x_{t,d}(v)) \right] \tag{as $v$ ranges over $L_B(w)$, $u$ also ranges over $L_B(w)$, as reasoned above }\\
        &= 2\sum_{v \in L_B(w)} \sum_{d \neq 0} \Ind\left[ g(x_{t,d}(v)) \neq h_{r,v}(x_{t,d}(v)) \right].
    \end{align*}
    We are now ready to conclude the desired error bound:
    \begin{align*}
        \frac{1}{q}\sum_{v \in L_B(w)} \err_{D_v, h_{r,v}}(g)&= \frac{1}{q}\sum_{v \in L_B(w)} \frac{1}{r(q-1)}\sum_{t \in [r]}\sum_{d \neq 0} \Ind\left[ g(x_{t,d}(v)) \neq h_{r,v}(x_{t,d}(v)) \right] \\
        &\ge  \frac{1}{rq(q-1)}\sum_{t \in B}\sum_{v \in L_B(w)}\sum_{d \neq 0} \Ind\left[ g(x_{t,d}(v)) \neq h_{r,v}(x_{t,d}(v)) \right] \\
        &\ge \frac{1}{rq(q-1)}\sum_{t \in B} \frac{q(q-1)}{2} \tag{using the calculation in the display above} \\
        &= \frac{|B|}{2r}=\frac{1}{4}.
    \end{align*}
\end{proof}

\subsection{The Monotone Adversary's Goal}

We will now construct a monotone adversary that seeks to exploit \Cref{claim:many-world-confusion-partial}. In particular, the adversary will aim to append carefully chosen samples to the clean dataset, so that the combined set of samples (post shuffling) makes it equally likely for the underlying true distribution to be $(D_v, h_{r,v})$ for any member $v$ within a $B$-line.

We first pin down the form of the final enlarged dataset that the adversary will target, and which causes this confusion for the learner. The strategy is as follows: for any $(x_{t,d}(w), h_{r,w}(x_{t,d}(w))$ in the clean dataset, the adversary will append a copy of $(x_{t,d}(v), h_{r,v}(x_{t,d}(v))$ for every $v \in L_B(w)$ within a careful chosen $B$-line. Then, a learner looking at this dataset does not which $(D_v, h_{r,v})$ is the clean distribution. 

More concretely, in what follows, set $r := 2n$, where $n$ is the clean dataset size. Let $\Fcal_n$ comprise of all possible ``frequency-vectors'' formed by a sample of $n$ points:
\begin{align*}
    \Fcal_n := \left\{f=(f_{t,d})_{t \in [r], d \in \F_{q} \setminus \{0\}}: f_{t,d} \in \Z_{\ge 0}, \sum_{t,d} f_{t,d} = n\right\}.
\end{align*}
For $f \in \Fcal_n$, let 
\begin{align*}
    \supp(f)=\left\{t:\sum_{d \neq 0} f_{t,d} > 0\right\}.
\end{align*}
Since $\sum_{t,d} f_{t,d} = n$, and $r=2n$, we have that $|\supp(f)| \le n$. Fix an ordering over $\Bcal_r$, and let
\begin{align*}
    \beta(f) := \min \{B \in \Bcal_r: B \cap \supp(f) = \emptyset\}.
\end{align*}
By our previous reasoning, $\beta(f)$ is well-defined for every $f \in \Fcal_n$.

Now, for any state $w$ and frequency vector $f \in \Fcal_n$, define the labeled (unordered) multiset
\begin{align}
    C_{w,f} := \biguplus_{t,d} f_{t,d} \cdot \{(x_{t,d}(w), h_{r,w}(x_{t,d}(w)))\}.
\end{align}
Here, the notation $f_{t,d} \cdot \{(x_{t,d}(w), h_{r,w}(x_{t,d}(w)))\}$ denotes a multiset having $f_{t,d}$ copies of the pair $(x_{t,d}(w), h_{r,w}(x_{t,d}(w)))$. Note that $|C_{w,f}|=n$.

Then, for $B=\beta(f)$, define:
\begin{align*}
    T_{w,f} = \biguplus_{v \in L_B(w)} C_{v,f}.
\end{align*}
We have that $|T_{w,f}|=qn$.

We now observe that $T_{w,f}$ is realizable by $h_{r,w}$, and is hence a valid monotone corruption. To see this, note that any point in $T_{w,f}$ is of the form $(x_{t,d}(v), h_{r,v}(x_{t,d}(v)))$ for $v \in L_B(w)$; furthermore, by definition of $B=\beta(f)$, for any $f_{t,d} > 0$, it holds that $t \notin B$. Therefore, by part(2) of \Cref{claim:same-labels-different-hypotheses-partial}, it holds that
\begin{align*}
    h_{r,w}(x_{t,d}(v)) = h_{r,v}(x_{t,d}(v)).
\end{align*}

Next, we observe that for any $v \in L_{B}(w)$, it holds that $L_{B}(v)=L_B(w)$; this immediately gives us that $T_{w,f}=T_{v,f}$ (as unordered multisets) for every $v \in L_{B}(w)$. This also means that, if $\Pi$ is a uniformly random permutation, $\Pi(T_{w,f})$ and $\Pi(T_{v,f})$ are identically distributed.


So, let $\Lcal$ be any learner, possibly using randomness $\sigma$. 
We have that
\begin{align}
    \frac{1}{q}\sum_{v \in L_B(w)} \E_{\Pi, \sigma} \left[\err_{\Dcal_v, h_{r,v}}(\Lcal(\Pi(T_{v,f})))\right]
    &= \E_{\sigma}\left[\frac{1}{q}\sum_{v \in L_B(w)} \E_{\Pi} \left[\err_{\Dcal_v, h_{r,v}}(\Lcal(\Pi(T_{v,f})))\right]\right] \tag{$\sigma$ and $\Pi$ are independent} \nonumber \\
    &= \E_{\sigma}\left[\frac{1}{q}\sum_{v \in L_B(w)} \E_{\Pi} \left[\err_{\Dcal_v, h_{r,v}}(\Lcal(\Pi(T_{w,f})))\right]\right] \tag{$\Pi(T_{v,f})$ and $\Pi(T_{w,f})$ are identically distributed for every $v \in L_B(w)$} \nonumber \\
    &= \E_{\sigma}\E_{\Pi}\left[\frac{1}{q}\sum_{v \in L_B(w)} \err_{\Dcal_v, h_{r,v}}(\Lcal(\Pi(T_{w,f})))\right] \nonumber \\
    &\ge \frac{1}{4}, \label{eqn:avg-error-bound}
\end{align}
where in the last step, we invoked \Cref{claim:many-world-confusion-partial}, since $\Lcal(\Pi(T_{w,f}))$ is a deterministic predictor conditioned on $\Pi$ and $\sigma$.

We now proceed to choose $w$ (which determines the $B$-line $L_B(w)$ that confuses the learner). The choice of $w$ will only depend on the learner and $n$, and will not depend on the realized sample/count vector $f$.

The following proposition provides us with the required choice of $w$:

\begin{proposition}[Choice of $w$]
    \label{prop:w-choice}
    Let $f$ be the frequency-vector of $n$ uniformly random draws from $[r] \times (\F_q \setminus \{0\})$. There exists $w^\star$, which depends only on $n$ and $\Lcal$, satisfying
    \begin{align}
        \E_{f}\left[E_{\Pi, \sigma} \left[\err_{D_{w^\star}, h_{r,w^\star}}(\Lcal(\Pi(T_{w^\star,f})))\right]\right] \ge \frac14.
    \end{align}
\end{proposition}
\begin{proof}
    From the above analysis, we know that for any fixed $w$ and $f$ (which determines $B=\beta(f)$), it holds that
    \begin{align*}
        \frac{1}{q}\sum_{v \in L_B(w)} \E_{\Pi, \sigma} \left[\err_{D_v, h_{r,v}}(\Lcal(\Pi(T_{v,f})))\right] \ge \frac14.
    \end{align*}
    We can sum the inequality above over the space $\Wcal_r$ of all possible states at level $r$, and observe that every fixed state $w$ gets counted in this sum exactly $q$ times (corresponding to the $q$ members of the $B$-line $L_B(w)$); this gives us that
    \begin{align*}
        \frac{1}{|\Wcal_r|}\sum_{w \in \Wcal_r} \E_{\Pi, \sigma} \left[\err_{D_w, h_{r,w}}(\Lcal(\Pi(T_{w,f})))\right] \ge \frac14.
    \end{align*}
    Averaging the above over $f$, we get that
    \begin{align*}
        \frac{1}{|\Wcal_r|}\sum_{w \in \Wcal_r} \E_f\left[\E_{\Pi, \sigma} \left[\err_{D_w, h_{r,w}}(\Lcal(\Pi(T_{w,f})))\right]\right] &= \E_f\left[\frac{1}{|\Wcal_r|}\sum_{w \in \Wcal_r} \E_{\Pi, \sigma} \left[\err_{D_w, h_{r,w}}(\Lcal(\Pi(T_{w,f})))\right]\right] \ge \frac14.
    \end{align*}
    In particular, the above implies the existence of a $w^\star \in \Wcal_r$ for which 
    \begin{align*}
        \E_f\left[\E_{\Pi, \sigma} \left[\err_{D_{w^\star}, h_{r,w^\star}}(\Lcal(\Pi(T_{w^\star,f})))\right]\right] \ge \frac14
    \end{align*}
    as desired.
\end{proof}

\subsection{Concluding the Lower Bound}
\label{sec:lower-bound-conclusion-partial}

We are now ready to fully specify the lower bound instance and the monotone adversary $\Acal$. Let $w^\star \in \Wcal_r$ be supplied by \Cref{prop:w-choice}, solely as a function of $\Lcal$ and $n$. Fix the target hypothesis and data distribution to be
\begin{align*}
    h^\star := h_{r, w^\star}, \qquad D := D_{w^\star}.
\end{align*}

Let $C=\{(x_i,y_i)\}_{i \in [n]}$ be a clean sample, where $x_1,\dots,x_n$ are i.i.d.\ draws from $D$, and $y_i = h^\star(x_i)$. The adversary $\Acal$ obtains the frequency-vector $f=(f_{t,d})_{t \in [r], d \in \F_q \setminus \{0\}}$ for each $x_{t,d}(w^\star)$\footnote{The adversary can do so, since $t$ is included explicitly in $x_{t,d}(w^\star)$, and $d$ may be extracted from $P=\{c_{w^\star, t}, c_{w^\star}+d\}$ (which is also included in $x_{t,d}(w^\star)$), the label $\lambda_P(c_{w^\star, t})$ and knowledge of the bijection $\lambda_P$.}. Note that $f$ is distributed as the frequency-vector of $n$ uniformly random draws from $[r] \times (\F_q \setminus \{0\})$. The adversary then computes $B = \beta(f)$, and appends the set
\begin{align*}
    U:=\biguplus_{v \in L_B(w^\star) \setminus w^\star} C_{v,f}
\end{align*}
to the clean sample $C$. Thus, the adversary appends exactly $n(q-1)$ corrupted points. As argued above, the enlarged dataset is consistent with the target hypothesis $h^\star=h_{r, w^\star}$. Moreover, the resulting dataset (in an unordered, multiset sense) is exactly $T_{w^\star, f}$. We therefore have, from \Cref{prop:w-choice}, that
\begin{align*}
    \E \, \err_{D, h^\star}(\Lcal(\Shuf(C \cup U)))
    &= \E_{f}\left[E_{\Pi, \sigma} \left[\err_{\Dcal_{w^\star}, h_{r,w^\star}}(\Lcal(\Pi(T_{w^\star,f})))\right]\right] \ge \frac14.
\end{align*}
This concludes the lower bound.

\section{Proofs for Proper Learning and ERM}\label{app:structured}
We prove the three results from \Cref{sec:structured} whose full arguments are
deferred here. Throughout, $S$ denotes the clean sample
of size $n$, and $T=\Shuf(S\cup U)$ the augmented sample, with multiplicities.
All learner randomness is private and independent of the adversary's choices.

\formatvspace{-0.3 cm}
\subsection{Arbitrarily Severe Proper-Learning Degradation}
\label{app:proper-blowup}

\formatvspace{-0.1 cm}
\paragraph{The construction.}
Let $L_1<L_2<\cdots$ be positive integers. Form disjoint finite blocks
\[
    X_k=P_k\sqcup C_k,
    \qquad P_k=\{p_{k,1},\ldots,p_{k,k}\},
    \qquad |C_k|=2L_k,
    \qquad \Xcal=\bigsqcup_{k\geq1}X_k.
\]
The set $P_k$ consists of markers, and $C_k$ is the core. Use public labels
$\star,\dollarlabel$ and distinct private labels $\lambda_{k,A},\rho_{k,i}$, each used by
only its indexed hypothesis. Define
\[
    c_{k,A}(x)=
    \begin{cases}
        \star,&x\in P_k\cup A,\\
        \lambda_{k,A},&x\in C_k\setminus A,\\
        \dollarlabel,&x\notin X_k,
    \end{cases}
    \qquad A\in\binom{C_k}{L_k},
\]
and
\[
    s_{k,i}(x)=
    \begin{cases}
        \rho_{k,i},&x=p_{k,i},\\
        \star,&x\in X_k\setminus\{p_{k,i}\},\\
        \dollarlabel,&x\notin X_k,
    \end{cases}
    \qquad i\in[k].
\]
These are the set hypotheses and fallbacks from
\eqref{eq:set} and \eqref{eq:fallback}. Write $\Hcal_k$ for their union on block $k$,
and set
\[
    \Hcal(L_1,L_2,\ldots)=\{h_{\dollarlabel}\}\cup\bigcup_{k\geq1}\Hcal_k,
    \qquad h_{\dollarlabel}\equiv\dollarlabel,
    \qquad |\Hcal_k|=\binom{2L_k}{L_k}+k.
\]
The domain, label space, and class are countable. At each $x\in X_k$, only
the finite family $\Hcal_k$ can output a label other than $\dollarlabel$, giving pointwise
finite range. For closedness in the product topology with discrete labels, let $f$
belong to the closure. If $f\neq h_{\dollarlabel}$, choose $x\in X_k$ with $f(x)\neq\dollarlabel$.
Only $\Hcal_k$ can agree there. If $f$ differs from each of its finitely many
members, one distinguishing coordinate per member, together with $x$, defines
a finite-coordinate neighborhood disjoint from the class, a contradiction.
Thus $f\in\Hcal_k$, proving closedness.

%
\paragraph{Two missed-mass estimates.}
For any fixed set $Q\subseteq\Xcal$,
\begin{equation}
    \P[S\text{ misses }Q]=(1-D(Q))^n\leq e^{-nD(Q)}.
    \label{eq:structured-miss}
\end{equation}
Since the augmented sample contains $S$, the same upper bound applies if it misses
$Q$. For a fixed finite family $\mathcal G$ and target $h^\star$, a union bound
over the disagreement sets gives
\begin{equation}
    \P\!\left[\exists g\in\mathcal G:
        g\text{ is consistent with }S,\
        \err_{D,h^\star}(g)>\epsilon\right]
    \leq|\mathcal G|e^{-n\epsilon}.
    \label{eq:structured-finite}
\end{equation}
This controls any consistent choice from $\mathcal G$, including one selected
using adaptive additions.

\paragraph{Upper bounds.}
Both learners begin with the same rule. A private label identifies the target
exactly. If all labels are $\dollarlabel$, return $h_{\dollarlabel}$. Otherwise, a $\star$-labeled example
identifies the target's block $k$. Let $g_k$ predict $\star$ on $X_k$ and
$\dollarlabel$ elsewhere. For $h^\star\in\Hcal_k$, its private-label region
$Q=\{x\in X_k:h^\star(x)\neq\star\}$ satisfies
\begin{equation}
    \err_{D,h^\star}(g_k)=D(Q),
    \qquad
    \err_{D,h^\star}(s_{k,i})\leq D(Q)+D(\{p_{k,i}\}).
    \label{eq:structured-errors}
\end{equation}
In an unresolved block, $Q$ was missed by the clean sample. An erroneous all-$\dollarlabel$
output likewise requires missing the target's entire block, of mass greater than
$\epsilon$, and has probability at most $e^{-n\epsilon}$. The target $h_{\dollarlabel}$ is
always learned perfectly.

\begin{lemma}\label{lem:adaptive-improper-upper}
For every sequence $(L_k)$, an improper learner under arbitrary adaptive additions
satisfies $\P[\err_{D,h^\star}(\widehat h)>\epsilon]\leq2e^{-n\epsilon}$.
\end{lemma}
\begin{proof}
In an unresolved block, return $g_k$. A bad output requires either missing the
whole block of mass greater than $\epsilon$ or missing $Q$ of mass greater than
$\epsilon$. Apply \eqref{eq:structured-miss} to these fixed sets.
\end{proof}

For the proper learner, fix a cutoff $K$. On an unresolved block $k\leq K$,
choose a member of $\Hcal_k$ consistent with the augmented sample; the target is
one such member. For $k>K$, choose a fallback. Put
$M_K=\max_{1\leq k\leq K}|\Hcal_k|$, and let $\beta$ bound the probability that
a large-block fallback is selected at a marker of mass greater than $\epsilon/2$.
Equations \eqref{eq:structured-finite} and \eqref{eq:structured-errors} give
\begin{equation}
    \P[\err_{D,h^\star}(\widehat h)>\epsilon]
    \leq(M_K+1)e^{-n\epsilon}+e^{-n\epsilon/2}+\beta.
    \label{eq:structured-proper-upper}
\end{equation}
The terms cover bad consistent outputs on small blocks or an all-$\dollarlabel$ output,
a missed private-label region of mass greater than $\epsilon/2$, and a heavy
selected marker. Each bound is taken in the original sample experiment, without
conditioning on entering a branch.

\begin{lemma}\label{lem:clean-proper-upper}
Suppose $L_k=k$ for $k\leq K_0=128$. At $\epsilon_0=1/16$, clean proper learning
uses $O(\log(1/\delta))$ examples, with a universal implicit constant.
\end{lemma}
\begin{proof}
Use cutoff $K_0=8/\epsilon_0$. In a large block, choose the least-observed of the
first $K_0$ markers, breaking ties by index. Let $q_i=D(\{p_{k,i}\})$ and let $Z_i$
be their clean counts. Some $i_0\in[K_0]$ has $q_{i_0}\leq1/K_0=\epsilon_0/8$.
For every $i$ with $q_i>\epsilon_0/2$, independence of the clean draws gives
\[
    \P[Z_i\leq Z_{i_0}]
    \leq\E\,2^{Z_{i_0}-Z_i}
    =\left(1+q_{i_0}-\frac{q_i}{2}\right)^n
    \leq e^{-n\epsilon_0/8}.
\]
The expectation assigns factors $2,1/2,1$ to a draw at $p_{k,i_0}$, at $p_{k,i}$,
or elsewhere. A heavy selected marker must have count at most $Z_{i_0}$, so
$\beta\leq K_0e^{-n\epsilon_0/8}$. Substituting into
\eqref{eq:structured-proper-upper} bounds failure by
$(M_{K_0}+K_0+2)e^{-n\epsilon_0/8}$. The first $K_0$ block sizes are fixed, so
$M_{K_0}$ and the resulting sample-complexity constant are universal.
\end{proof}

%
\paragraph{The oblivious proper lower bound.}
\begin{lemma}\label{lem:proper-finite-scale}
For $k\geq7$ and $1\leq n\leq L_k/(2k)$, every possibly randomized proper learner
has failure probability at least $1/k$ at error threshold $1/16$ on some instance
with a deterministic oblivious adversary making exactly $(k-1)n$ additions.
\end{lemma}
\begin{proof}
Write $L=L_k$. Choose a uniform $L$-subset $A\subseteq C_k$ and an independent
uniform $J\in[k]$. Use target $c_{k,A}$ and marginal
\[
    D_{J,A},\qquad
    D_{i,A}=\tfrac12\delta_{p_{k,i}}+\tfrac12\Unif(A),\qquad i\in[k],
\]
where $\delta_x$ denotes a point mass. All clean labels are $\star$. For each
$i\neq J$, the adversary adds an independent $n$-sample from $D_{i,A}$, labeled
$\star$. These additions are correct and may be drawn before $S$.

Let $W_J$ be the clean batch and $W_i$ the added batch for $i\neq J$. Conditional
on $A$, the ordered batch vector $W=(W_1,\ldots,W_k)$ has independent coordinates
$W_i\sim D_{i,A}^n$, regardless of $J$. Thus $J$ is uniform and independent of
$(A,W)$, even if the learner is shown the batch partition.

Let $V$ be the distinct observed core points in all batches, and put
$v=|V|\leq kn\leq L/2$. Conditional on $W$, the posterior law of $A$ is uniform
among the $L$-subsets containing $V$: if $r$ counts the core observations with
multiplicity, every compatible set has likelihood $2^{-kn}L^{-r}$.
Condition also on the independent shuffle and learner randomness, fixing the
output without changing either posterior law.

If the learner outputs $c_{k,B}$, then $|B|=L$, but consistency is not required.
A point in $V$ belongs to $A$ with probability one, and any unobserved core point
belongs to $A$ with probability
\[
    \frac{L-v}{2L-v}\geq\frac14.
\]
Since $C_k\setminus B$ contains $L$ points,
$\E[|A\setminus B|\mid W,\text{seeds}]\geq L/4$. Writing
$p=\P[|A\setminus B|>L/8\mid W,\text{seeds}]$, the bound
$L/4\leq(1-p)L/8+pL$ gives $p\geq1/7$. This output's population error is
$|A\setminus B|/(2L)$, so it exceeds $1/16$ with probability at least $1/7$.

A fallback $s_{k,i}$ has error $1/2$ exactly when $i=J$, an event of conditional
probability $1/k$. The constant hypothesis and all hypotheses from other blocks
have error one. Every proper output therefore fails with probability at least
$1/k$, since $k\geq7$.

Once $(A,J)$ is fixed, the added batches are independent of the clean batch.
Averaging first over $(A,J)$ and these added batches fixes a target, marginal,
and deterministic oblivious addition multiset attaining at least the same failure
probability over the clean sample, shuffle, and learner randomness.
\end{proof}

\begin{proofof}{\Cref{thm:arbitrary-proper-blowup}}
Set $L_k=k$ for $k\leq K_0=128$, and for $k>K_0$ set
\[
    L_k=\max\left\{L_{k-1}+1,\
        2k\left\lceil F(2(k+1))\right\rceil\right\}.
\]
Let $\Hcal_F=\Hcal(L_1,L_2,\ldots)$. The clean proper and adaptive improper upper
bounds follow from \Cref{lem:clean-proper-upper,lem:adaptive-improper-upper}.

For sufficiently small $\delta$, take $k=\lfloor1/(2\delta)\rfloor>K_0$. Then
$1/k\geq2\delta>\delta$, $1/\delta<2(k+1)$, and
\[
    \left\lfloor\frac{L_k}{2k}\right\rfloor
    \geq\left\lceil F(2(k+1))\right\rceil
    \geq F(1/\delta).
\]
By \Cref{lem:proper-finite-scale}, no integer sample size up to the left side
suffices at confidence $\delta$. This gives the claimed proper oblivious lower
bound, using $(k-1)n=O(n/\delta)$ additions.

For the matching logarithmic lower bounds, give equal mass to a core point
$a\in C_2$ and $p_{2,1}$, and compare targets $s_{2,1},s_{2,2}$. With probability
$2^{-n}$ the sample consists only of $(a,\star)$ and has the same law for both targets.
Their labels differ at $p_{2,1}$, so averaging over the targets forces population
error at least $1/2$ with probability at least $2^{-(n+1)}$ for some target,
even for an unrestricted randomized learner. Thus $n=\Omega(\log(1/\delta))$
is necessary; the empty adversary gives the same lower bound in the adaptive
model. All constants are independent of $F$.
\end{proofof}

Fixing the initial blocks as above and taking $L_k=2^k$ thereafter gives an
exponential confidence separation:
\[
    N^{\prop}_{\obl}(\epsilon_0,\delta)=\exp(\Theta(1/\delta)).
\]
For a sequence with $L_k/k\to\infty$, put
$\kappa(n)=\min\{k\geq7:L_k\geq2kn\}$. The finite-scale bound gives
$R^{\prop}_{\obl}(n)\geq1/(16\kappa(n))$, while integrating the tail bound in
\Cref{lem:adaptive-improper-upper} gives $R^{\imp}_{\ad}(n)=O(1/n)$.
Choosing rapidly growing block sizes makes the proper lower bound decay
arbitrarily slowly.

\subsection{ERM Fragility}\label{app:erm}

\formatvspace{-0.3 cm}
\paragraph{A DS-one class whose ERMs are fragile.}
\label{app:erm-fragile}
Take $\Hcal=\Hcal(1,2,3,\ldots)$ from the preceding construction,
so block $k$ has $k$ markers and a $2k$-point core. Countability, closedness, and
pointwise finite range have already been proved.

\begin{proofof}{\Cref{thm:erm-fragile}}
\textbf{DS dimension.}
No two-coordinate pseudo-cube can contain a private label at either coordinate:
only one hypothesis uses that label, so no neighbor can preserve it and change
the other coordinate. Only the public labels $\star,\dollarlabel$ remain. On two points in
the same block, the only public patterns are $(\star,\star)$ and $(\dollarlabel,\dollarlabel)$; on two
points in different blocks, $(\star,\star)$ is absent. Neither configuration contains
the four patterns required by a binary two-coordinate pseudo-cube. Projecting
a higher-dimensional pseudo-cube onto two coordinates preserves the neighbor
property, so the DS dimension is at most one. Two hypotheses differ at a single
point, proving that it is exactly one.

\textbf{A clean ERM.}
Return the identified target if a private label appears, and $h_{\dollarlabel}$ if every
label is $\dollarlabel$. Otherwise the block $k$ is known. If some marker is unseen,
return the fallback at the least-indexed unseen marker; if all markers have
appeared, return any consistent set hypothesis, with fixed tie-breaking.
The fallback fits every observed $\star$ label and every outside-block $\dollarlabel$ label. In the
second case the target cannot be a fallback, so a consistent set hypothesis
exists. Extend the rule on nonrealizable inputs by any fixed empirical risk
minimizer. The resulting rule is a deterministic ERM on every input.

Fix $\epsilon\in(0,1)$ and $n\geq16/\epsilon$. The all-$\dollarlabel$ branch fails with
probability at most $e^{-n\epsilon}$. For a target in block $k\leq\epsilon n/4$,
the number of possible active-block outputs is
$\binom{2k}{k}+k\leq2\cdot4^k\leq2e^{2k}$. By
\eqref{eq:structured-finite}, a bad consistent output has probability at most
$2e^{2k-n\epsilon}\leq2e^{-n\epsilon/2}$.

For $k>\epsilon n/4$, at least $k/2$ markers have mass at most $2/k$. Let $U$ count
the unseen markers in a fixed group of these markers. Since $k>4$ and
$\log(1-u)\geq-2u$ for $0\leq u\leq1/2$,
\[
    \E U\geq\frac{k}{2}(1-2/k)^n\geq\frac{k}{2}e^{-4n/k}.
\]
For two distinct markers of masses $q_i,q_j$, their absence indicators satisfy
\[
    \E[I_iI_j]=(1-q_i-q_j)^n
    \leq(1-q_i)^n(1-q_j)^n=\E I_i\,\E I_j.
\]
Hence $\operatorname{Var}(U)\leq\E U$, and Chebyshev gives
\begin{equation}
    \P[\text{every marker appears}]
    \leq\P[U=0]\leq\frac1{\E U}
    \leq\frac{2e^{4n/k}}{k}
    <\frac{8e^{16/\epsilon}}{\epsilon n}.
    \label{eq:structured-occupancy}
\end{equation}
If a marker is unseen and the target remains unidentified, the returned fallback
has error at most $D(Q)+D(\{p_{k,i}\})$ by \eqref{eq:structured-errors}.
The first term exceeds $\epsilon/2$ in this branch only if a fixed region of
that mass was missed, an event of probability at most $e^{-n\epsilon/2}$.
At most $2/\epsilon$ markers have mass greater than $\epsilon/2$, and each is missed
with probability at most $e^{-n\epsilon/2}$. A union bound controls every heavy
unseen marker, including one chosen using the sample. Combining the cases yields
\[
    \P[\err_{D,h^\star}(\widehat h)>\epsilon]
    \leq(4+2/\epsilon)e^{-n\epsilon/2}
        +\frac{8e^{16/\epsilon}}{\epsilon n}.
\]
For each fixed $\epsilon$, this tends to zero uniformly over targets and marginals,
proving clean PAC learnability.

\textbf{Oblivious lower bound for every ERM.}
Fix $\alpha>0$ and take $n$ large enough that $k=\lfloor\alpha n\rfloor\geq2$.
Choose a uniform $k$-subset $A\subseteq C_k$, use target $c_{k,A}$, and take
$D_A=\Unif(A)$. The adversary adds one $\star$-labeled copy of each marker,
independently of the clean draw. These $k\leq\alpha n$ correct additions eliminate
every fallback; the $\star$ labels also eliminate $h_{\dollarlabel}$ and every other block.
An ERM must therefore output $c_{k,B}$ for a $k$-subset $B\supseteq V$, where $V$
is the distinct observed core set.

Conditional on the ordered clean sample, $A$ is uniform among the $k$-subsets
containing $V$: each compatible set has sample likelihood $k^{-n}$. The additions
contain no further information about $A$. Conditioning also on independent
shuffle and learner randomness fixes $B$ without changing this posterior.
Writing $v=|V|$, each of the $k$ excluded points in $C_k\setminus B$ belongs to
$A$ with probability $(k-v)/(2k-v)$. Consequently,
\[
    \E[|A\setminus B|\mid\text{sample, seeds}]
    =\frac{k(k-v)}{2k-v}\geq\frac{k-v}{2}.
\]
The population error is $|A\setminus B|/k$, and
$\E[k-v]=k(1-1/k)^n$, since each target point is missed with that probability.
Thus
\[
    \E\,\err_{D_A,c_{k,A}}(c_{k,B})
    \geq\frac12(1-1/k)^n
    \longrightarrow\frac12e^{-1/\alpha},
\]
where $n/k\to1/\alpha$ and $n/k^2\to0$. Averaging fixes a hard $A$ for each $n$,
giving the stated asymptotic expected-error lower bound. A PAC guarantee would
imply expected error at most $\epsilon+\delta$, contradicting this positive lower
bound for sufficiently small fixed $\epsilon,\delta$.
\end{proofof}

\begin{remark}
By \Cref{cor:oblivious-free}, the same class admits an improper learner with
expected error at most $1/(n+1)$ against an arbitrary number of oblivious
additions. %
The ERM lower bound therefore coexists with optimal unrestricted learning.

\end{remark}

\formatvspace{-0.6 cm}
\paragraph{One adaptive addition can degrade every ERM.}
\label{app:one-addition-erm}

\begin{proposition}
\label{prop:one-addition-erm-full}
There is a binary class of VC dimension two such that one adaptive monotone
addition can force every possibly randomized ERM to incur expected error
$\Omega(\log n/n)$.
\end{proposition}
\formatvspace{-0.6 cm}
\begin{proof}
%
Let $\Xcal=\binom{\N}{2}$ and
\[
    h_0(e)=0,\qquad h_i(\{u,v\})=\1\{i\in\{u,v\}\}\quad(i\geq1),
    \qquad \Hcal=\{h_0,h_1,h_2,\ldots\}.
\]
Two adjacent edges $\{1,2\},\{1,3\}$ are shattered by $h_0,h_1,h_2,h_3$.
For three edges without a common vertex, the pattern $111$ is absent.
If they share a vertex, they form a star: its center realizes $111$, each leaf
realizes a single one, and all other hypotheses realize $000$. No pattern with
exactly two ones occurs. The VC dimension is therefore two.

Fix a sufficiently large $n$, put $N=\lceil8n/\log n\rceil$ and $M=N-1$, and
choose a uniform target vertex $I\in[N]$. Under target $h_I$, take $D_I$ uniform
on the $\binom M2$ edges within $[N]\setminus\{I\}$. Every clean label is zero,
so $h_0$ has zero population error on these instances.

For an ordered clean sample $s$, let $Z(s)\subseteq[N]$ be the vertices appearing
in no observed edge. The sample likelihood under target $i$ is $\binom M2^{-n}$
if $i\in Z(s)$ and zero otherwise. Hence $I\mid S=s\sim\Unif(Z(s))$.
If $|Z(s)|\geq2$, choose $J$ uniformly from $Z(s)\setminus\{I\}$ and add
$(\{I,J\},1)$. This is a correct example, and only $h_I,h_J$ remain consistent.
Conditional on $s$, every ordered pair of distinct endpoints $(I,J)$ has
probability $1/(|Z(s)|(|Z(s)|-1))$. After observing the unordered positive edge,
the target is therefore equally likely to be either endpoint. The shuffled input
has the same law in both cases, so every ERM, including a randomized one, chooses
the wrong endpoint with probability $1/2$. If $|Z(s)|=1$, add $(\{I,J\},1)$
for any fixed $J\in[N]\setminus\{I\}$; only $h_I$ then remains consistent.

Let $U$ count the unseen nontarget vertices, so $|Z(s)|=U+1$. A clean edge touches
each such vertex with probability $2/M$, giving $\E U=M(1-2/M)^n$.
For two distinct nontarget vertices, the probability that one edge avoids both is
\[
    \frac{\binom{M-2}{2}}{\binom M2}
    =\frac{(M-2)(M-3)}{M(M-1)}
    \leq(1-2/M)^2.
\]
Raising to the $n$th power shows that the absence indicators have nonpositive
covariance. Thus $\operatorname{Var}(U)\leq\E U$ and $\P[U=0]\leq1/\E U$.
For sufficiently large $n$, $4n/\log n\leq M-2\leq M\leq8n/\log n$. Using
$\log(1-x)\geq-x/(1-x)$ for $0\leq x<1$ gives
\[
    \E U\geq M\exp\!\left(-\frac{2n}{M-2}\right)
    \geq\frac{4\sqrt n}{\log n}\longrightarrow\infty.
\]
In particular, $\P[U>0]\geq1/2$ for all sufficiently large $n$.

If the target is $h_I$ and the ERM returns $h_J$ with $J\neq I$, its error is
the fraction of nontarget edges incident to $J$, namely $(M-1)/\binom M2=2/M$.
The randomized experiment therefore has expected error
\[
    \frac12\cdot\frac2M\P[U>0]
    \geq\frac1{2M}\geq\frac{\log n}{16n}.
\]
Averaging fixes a target $I$ and its marginal with at least this expected error.
For each clean sample with $|Z(s)|\geq2$, replace the random neighbor by one
maximizing conditional expected error over the learner's private randomness and
the final shuffle. The maximum over this nonempty finite set is at least the
average, giving a deterministic adaptive adversary with no smaller expected error.
The $|Z(s)|=1$ branch already has a fixed choice. Every branch adds exactly one
correct example, and the same countable class works for all these sample sizes.

\end{proof}

\end{appendices}

\end{document}